\documentclass[12pt]{article}

\usepackage{scicite}

\usepackage{times}

\usepackage{booktabs}
\usepackage{xcolor}
\usepackage[colorlinks = true,
            linkcolor = blue,
            urlcolor  = blue,
            citecolor = blue,
            anchorcolor = blue]{hyperref}

\usepackage{graphicx}
\usepackage{amsmath}
\usepackage{amsfonts}
\usepackage{bm}
\def\x{\bm{x}}
\def\s{\bm{s}}
\def\n{\bm{n}}
\def\y{\bm{y}}
\usepackage{amsthm}
\newtheorem*{theorem}{Noise2Sim Theorem}
\usepackage{diagbox}
\usepackage{multirow}

\newenvironment{sciabstract}{%
\begin{quote} \bf}
{\end{quote}}

\title{Suppression of Correlated Noises with Similarity-based Unsupervised Deep Learning}

\author
{Chuang Niu,$^{1}$ Mengzhou Li,$^{1}$ Fenglei Fan,$^{1}$ Weiwen Wu,$^{1}$ Xiaodong Guo,$^{1}$\\
Qing Lyu,$^{1}$ Ge Wang$^{1\ast}$\\
\\
\normalsize{$^{1}$Department of Biomedical Engineering, Rensselaer Polytechnic Institute,}\\
\normalsize{110 8th Street, Troy, New York 12180, USA}\\
\\
\normalsize{$^\ast$To whom correspondence should be addressed; E-mail:  wangg6@rpi.edu.}
}

\date{}

\begin{document} 


\baselineskip24pt


\maketitle


\begin{sciabstract}
Image denoising is a prerequisite for downstream tasks in many fields. Low-dose and photon-counting computed tomography (CT) denoising can optimize diagnostic performance at minimized radiation dose. Supervised deep denoising methods are popular but require paired clean or noisy samples that are often unavailable in practice. Limited by the independent noise assumption, current unsupervised denoising methods cannot process correlated noises as in CT images. Here we propose the first-of-its-kind similarity-based unsupervised deep denoising approach, referred to as Noise2Sim, that works in a nonlocal and nonlinear fashion to suppress not only independent but also correlated noises. Theoretically, Noise2Sim is asymptotically equivalent to supervised learning methods under mild conditions. Experimentally, Nosie2Sim recovers intrinsic features from noisy low-dose CT and photon-counting CT images as effectively as or even better than supervised learning methods on practical datasets visually, quantitatively and statistically. Noise2Sim is a general unsupervised denoising approach and has great potential in diverse applications.
\end{sciabstract}


\section*{Introduction}

Computed tomography (CT) reconstructs cross-sectional or volumetric images from many X-ray projections taken at different angles, and is a widely used diagnostic tool around the world.
As an emerging CT technology, photon-counting CT (PCCT) is being actively developed with major advantages including rich tissue contrast, high spatial resolution, low radiation dose, and tracer-enhanced K edge imaging \cite{pcct}.
In the medical CT field, radiation dose must be minimized according to the "As Low As Reasonably Achievable" (ALARA) guideline~\cite{ldctrisk}.
Not surprisingly, reduced radiation dose will degrade the CT image quality and affect the diagnostic performance.
In particular, PCCT uses multiple energy windows and much reduced detector sizes. As a result, each line integral measurement can only be made with a significantly lower number of x-ray photons.
Over the past decade, great efforts have been made on low-dose CT (LDCT) reconstruction, also known as LDCT denoising, which has now a boosted momentum due to the recent FDA-approval of the PCCT technology \cite{pcct}.

Image denoising is to recover signals hidden in a noisy background.
Since noise is a statistical fluctuation governed by quantum mechanics, denoising is generally achieved by a mean/averaging operation. For example, local averaging methods include  Gaussian smoothing~\cite{LINDENBAUM19941},  anisotropic filtering\cite{56205,2158083}, neighborhood filtering~\cite{19870670317,Smith1997,710815}, and transform domain processing \cite{dct}. On the other hand, nonlocal averaging methods use various nonlocal means with Gaussian kernel based weighting~\cite{nlm} or via nonlocal collaborative filtering in a transform domain~\cite{pnas1, bm3d}. Impressively, the nonlocal methods usually outperform the local methods, as images usually consist of repeated features or patterns that can be leveraged over a field of view to recover signals coherently.
Over the past several years, the area of image denoising has been dominated by deep convolutional neural networks (CNNs) \cite{7839189}. Different from the traditional methods that directly denoise an image based on an explicit model, the deep learning approach optimizes a deep neural network using training data, and then uses the trained model to predict a denoised image, usually achieving better results with much less inference time than the traditional denoising methods.

The mainstream deep denoising methods \cite{TIAN2020251} require paired and registered noise-clean images to train the networks, denoted as Noise2Clean.
However, such noise-clean image pairs can be hardly obtained in real-world applications.
To relax the requirement of paired noise-clean samples, Noise2Noise~\cite{n2n} was proposed to train a denoising network with paired noise-noise images that share the same content but are instantiated with independent noises.
It has been proved that Noise2Noise-based training can optimize a network to approach the Noise2Clean quality under the assumption of zero-mean noise.
However, the collection of paired noisy images is still expensive in many scenarios.
To address this challenge, extensive efforts~\cite{n2v, noise2self, noise2same, fcomp, snt, NIPS2019_8920, 9098336, 8579082, s2s} were made to develop unsupervised learning methods that train a denoising network with single noisy images.
Along this direction, Noise2Void~\cite{n2v} and Noise2Self~\cite{noise2self} were proposed to predict each center pixel from its local neighbors, achieving promising results using single noisy images under the assumption of {\it independent} noises among neighbor pixels.
However, these unsupervised deep denoising methods cannot suppress correlated or structured noises. Since most noises, especially CT image noises, are correlated, a general unsupervised denoising approach is highly desirable to unleash the power of deep learning for effective suppression of both independent and correlated noises, which is the holy grail of denoising methods.

Symmetry including similarity is ubiquitous in the physical world, naturally reflected in the image domain, featured by repeated or recursively embedded structures, edges, textures, etc., and plays an important role in modern science, engineering and medicine \cite{selfsimilarity}.
However, there is currently no general approach to utilize similar features in images for deep learning-based denoising.
As the first effort to utilize similarity for unsupervised deep denoising, the initial version of our Noise2Sim method was shared on arXiv~\cite{noise2sim}. Somehow similar methods were heuristically designed for the same purpose~\cite{noise2context}.
Here we present our Noise2Sim approach as the general framework for similarity-based optimization of a deep denoising network using intrinsically-registered sub-images.
Importantly, the theorem is proved that under mild conditions the deep denoising network trained with Noise2Sim is asymptotically equivalent to that trained in the supervised learning mode.
Given the equivalency between Noise2Clean and Noise2Sim, the proposed unsupervised learning approach can, in principle, suppress both independent and correlated noises as effectively as the supervised deep denoising methods \cite{shanldct, nmdenoising}.

The emphasis of this paper is on the general unsupervised deep denoising approach and its practical application in LDCT and PCCT imaging where image noises are correlated and practically paired samples are usually not available.
Extensive experiments on 3D LDCT images and 4D PCCT images consistently show the superiority of Noise2Sim over both the traditional denoising methods and the state-of-the-art deep denoising networks.
Furthermore, we systematically analyze the effectiveness of Noise2Sim on denoising natural images with independent noises in section S1 the supplementary material.
A recent study has shown that Noise2Sim can be successfully applied to other domains as well \cite{noise2sim_seismic}.
The source code, pre-trained models, and all data used in this study have been made publicly available on the project page \url{http://chuangniu.info/projects/noise2im/}, where the Python package "noise2sim" can be easily installed through Pip and source.

\section*{Results}

\noindent \textbf{Superiority/competitiveness of Noise2Sim relative to supervised deep LDCT denoising.}
Given the well-known successes of deep imaging \cite{deepimaging}, deep learning-based methods recently became dominating in denoising LDCT images \cite{redcnn}. 
In particular, Shan \emph{et al.} demonstrated that the deep learning approach performs better than or comparable to the commercial iterative reconstruction methods in a double blind reader study \cite{shanldct}.
These methods work mainly in the supervised learning mode, assuming that the paired yet registered low-dose and normal-dose CT (NDCT) images are available.
However, supervised LDCT image denoising requires either simulated noise insertion which does not reflect all the technical factors of a particular LDCT scan or the repeated scans of the same patient which are impractical in a clinical setting.
To address this challenge, Noise2Sim can be applied to denoise LDCT images directly, without any clean or noisy labels.

We first evaluated Noise2Sim on the commonly used Mayo clinical dataset containing LDCT and NDCT image pairs, demonstrating that Noise2Sim without using any annotation is comparable to and sometimes even better than supervised learning methods.
Popular supervised deep LDCT denoising methods RED-CNN~\cite{redcnn} and MAP-NN~\cite{shanldct} were selected for comparison.
Noise2Clean that is the plain supervised version of Noise2Sim serves as the baseline.
Also, we selected BM3D \cite{bm3d} and Nosie2Void \cite{n2v} for comparison, which are the most popular traditional denoising method and the representative unsupervised deep denoising method respectively.
Note that BM3D has the best performance among the traditional denoising methods as demonstrated in \cite{redcnn} for LDCT denoising.
Both visual results and quantitative Structural Similarity Index (SSIM) and Peak signal-to-noise ratio (PSNR) results were used to evaluate the denoising performance.

\begin{table}[htp]
\scriptsize
  \renewcommand{\arraystretch}{1.5}
  \renewcommand\tabcolsep{4.5pt}
 \caption{Quantitative results of different LDCT denoising methods on the Mayo dataset, in terms of mean and standard deviation values over the test dataset (1136 slices). N2Sim- means dissimilar pixels are excluded during training, N2Sim-L1 means using L1 loss, and N2Sim* means directly training the model on the testing dataset.}
  \centering
  \begin{tabular}{c|ccc|cccccc}
                  & \multicolumn{3}{c|}{Supervised learning} & \multicolumn{6}{c}{Unsupervised learning}  \\ 
    \hline
    Methods       & RED-CNN & MAP-NN & N2Clean &  N2Void & BM3D & N2Sim- & N2Sim & N2sim-L1 & N2Sim*\\
     \midrule
    PSNR          & 28.58 \tiny{$\pm$ 1.54} & 28.28 \tiny{$\pm$ 1.55} & \textbf{28.78} \tiny{$\pm$ 1.58} & 23.36 \tiny{$\pm$ 1.82} & 27.28 \tiny{$\pm$ 1.48} & 27.68 \tiny{$\pm$ 1.31} & 28.30 \tiny{$\pm$ 1.48} & \textbf{28.38} \tiny{$\pm$ 1.51} & 28.33 \tiny{$\pm$ 1.50} \\
    SSIM          & \textbf{90.30} \tiny{$\pm$ 2.92} & 90.13 \tiny{$\pm$ 2.92} & 90.21 \tiny{$\pm$ 2.91} & 83.99 \tiny{$\pm$ 4.33} & 88.30 \tiny{$\pm$ 3.02} & 89.73 \tiny{$\pm$ 2.73} & 90.33 \tiny{$\pm$ 2.78} & \textbf{90.45} \tiny{$\pm$ 2.76} & 90.37 \tiny{$\pm$ 2.89} \\
    \bottomrule
  \end{tabular}
  \label{tab:ldct}
\end{table}

The quantitative results in Table~\ref{tab:ldct} show that Noise2Sim achieves PSNR measures comparable to and better SSIM values than the supervised learning methods.
Indeed, our paired t-Test indicates that the \emph{p} values are less than 0.001 for the zero average difference hypothesis; that is, the mean results from different methods are significantly different.
Therefore, the mean results can be reliably used to compare different methods.
The above results suggest that even if the paired training samples are not available, deep neural networks can sill be optimized and applied to improve LDCT quality up to the state-of-the-art supervised performance.
Additionally, we evaluated the effects of three factors on Noise2Sim, including dissimilar pixels, loss functions, and training sets, which are detailed in Methods. The default learning setting, denoted by Noise2Sim, identifies and excludes dissimilar pixels in training samples, uses the mean squared error (MSE) loss function, and performs training and testing on separate datasets. In Table \ref{tab:ldct}, we tested three variants of Noise2Sim, i.e., Noise2Sim- that includes dissimilar pixels, Noise2Sim-L1 that uses the L1 loss function, and Noise2Sim* that trains the denoising model using the LDCT images in the testing dataset. The denoising performance of Noise2Sim- is evidently worse than Noise2Sim as the zero-mean conditional discrepancy condition required by the Nosie2Sim theorem is violated if the dissimilar pixels are not excluded. On the Mayo dataset, the L1 loss works better than the MSE loss, and the difference between the L1 and L2 losses is further analyzed in Methods.
More interestingly, the model trained using a small set of testing images performs slightly better than that trained with a much larger training set, which means that we do not need to collect a large number of extra training samples.
In this case, however, the time required for denoising includes both the training and testing models, and thus is much longer than that needed to directly apply the model pre-trained on the training dataset.
Either training strategy for Noise2Sim can be selected according to the application scenarios.

\begin{figure}[h]
    \centering
    \includegraphics[width=0.99\textwidth]{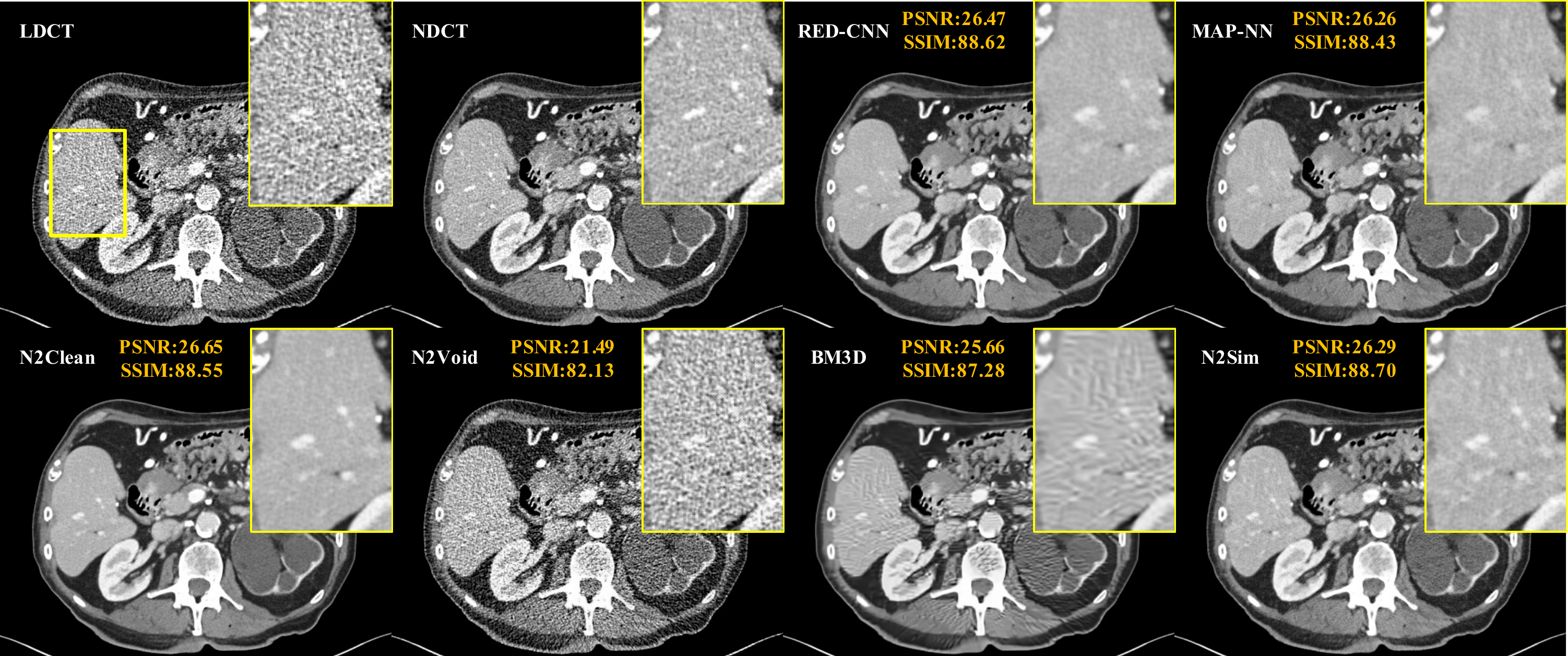}
    \caption{Superior/comparative Noise2Sim results on the Mayo LDCT dataset.  The yellow ROIs indicate that detail structures are better preserved using Noise2Sim than other methods in reference to the normal-dose CT image. The display window is [-160, 240] in Hounsfield unit~(HU), along with PSNR and SSIM values.}
    \label{fig:ldctresults}
\end{figure}

The visual results in Fig.~\ref{fig:ldctresults} also illustrate that Noise2Sim is better than the supervised methods in preserving structural details, as indicated by the zoomed ROIs in the yellow bounding boxes, in reference to the normal-dose CT image. Clearly, the supervised methods tend to remove noises aggressively while Noise2Sim tries to preserve informative details.
Our results show that Noise2Void does not work for structured CT noises, and the BM3D results are either over-smoothed or compromised with structured artifacts.
In contrast, Noise2Sim has significantly better performance in reducing structured noises and preserving content structures.

\begin{figure*}[h]
    \centering
    \includegraphics[width=0.99\textwidth]{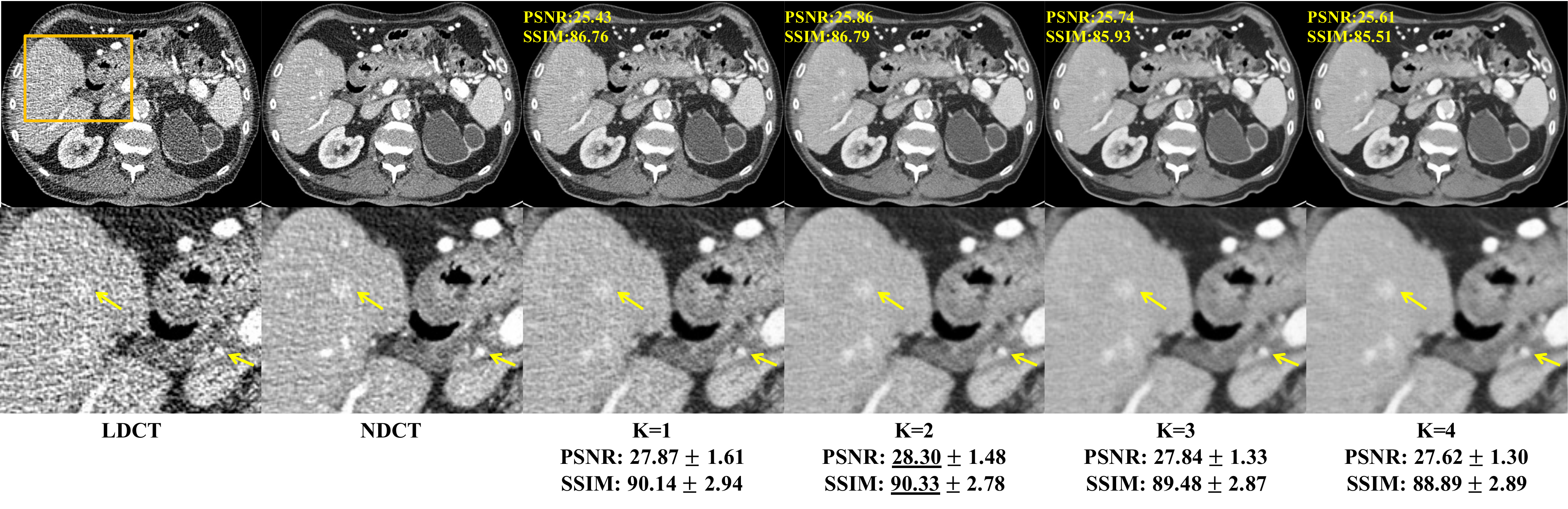}
    \caption{Controllable Noise2Sim results with different denoising levels on the Mayo dataset. The numbers in yellow and black are calculated for the individual case. The numbers in black are respectively the mean and standard deviation values over the test dataset (1136 slices). The second row shows the corresponding ROIs in the yellow bounding boxes. The yellow arrows indicate that structural details are enhanced via image denoising.}
    \label{fig:ldctresults-k}
\end{figure*}

\noindent \textbf{Controllability of Noise2Sim with different denoising levels keyed to a single parameter.}
The similarity parameter $k$ defined in the Methods section allows us to control the extent to which image noise is removed.
As shown in Fig \ref{fig:ldctresults-k}, more noise can be suppressed with a larger parameter $k$.
A larger $k$ increases the noise independence of similar training samples so that noise can be suppressed more aggressively, as implied by the Noise2Sim theorem.
On the other hand, the denoising results may be harmed if $k$ is too large, as the zero-mean conditional discrepancy assumed by the Noise2Sim theorem may be compromised. 
The statistical results in Fig \ref{fig:ldctresults-k} shows that $k=2$ achieves the best trade-off between noise independence and feature similarity.
The parameter $k$ can be adjusted according to specific down-stream tasks. If image quality can be quantitatively modeled, such as with a neural network and/or a Gram matrix~\cite{gram}, $k$ could be automatically optimized.
More practically, several levels of denoising images can be simultaneously presented to radiologists, and thus the best image quality can be determined with the human expertise in loop.

\begin{table}[htp]
\scriptsize
  \renewcommand{\arraystretch}{1.5}
  \renewcommand\tabcolsep{4.5pt}
 \caption{Generalizability results of different denoising methods on FDA datasets in terms of PSNR and SSIM, with both mean and standard deviation values over the test dataset (408 slices).}
  \centering
  \begin{tabular}{cc|ccc|ccc}
    \multicolumn{2}{c|}{\multirow{2}{*}{Method}}     &         \multicolumn{3}{c|}{Supervised learning} & \multicolumn{3}{c}{Unsupervised learning}     \\ 
    \cline{3-8}
    \multicolumn{2}{c|}{}       & RED-CNN & MAP-NN & Noise2Clean &  Noise2Void & BM3D & Noise2Sim*\\
    \midrule
    \multirow{2}{*}{\emph{b40f}} & PSNR          & 30.37 \tiny{$\pm$ 2.46} & 30.19 \tiny{$\pm$ 2.30} & 30.44 \tiny{$\pm$ 2.34} & 23.43 \tiny{$\pm$ 2.76} & 27.28  \tiny{$\pm$ 1.28}  & \textbf{30.52} \tiny{$\pm$ 2.12} \\
                                 & SSIM          & 91.70 \tiny{$\pm$ 3.95} & 91.28 \tiny{$\pm$ 4.11} & 91.66 \tiny{$\pm$ 3.88} & 83.10 \tiny{$\pm$ 5.22} & 88.31  \tiny{$\pm$ 3.02} & \textbf{92.03} \tiny{$\pm$ 3.79} \\
    \midrule
    \multirow{2}{*}{\emph{b60f}} & PSNR          & 19.59 \tiny{$\pm$ 2.43} & 22.99 \tiny{$\pm$ 0.81} & 17.71 \tiny{$\pm$ 1.75} & 17.05 \tiny{$\pm$ 1.02} & 21.21 \tiny{$\pm$ 2.01}  & \textbf{26.46} \tiny{$\pm$ 0.78}\\
                                 & SSIM          & 83.36 \tiny{$\pm$ 4.16} & 81.79 \tiny{$\pm$ 4.79} & 82.20 \tiny{$\pm$ 0.03} & 72.42 \tiny{$\pm$ 4.11} & 80.53 \tiny{$\pm$ 4.95}  & \textbf{91.80} \tiny{$\pm$ 1.28}\\
    \bottomrule
  \end{tabular}
  \label{tab:fda}
\end{table}

\begin{figure}[h]
    \centering
    \includegraphics[width=0.99\textwidth]{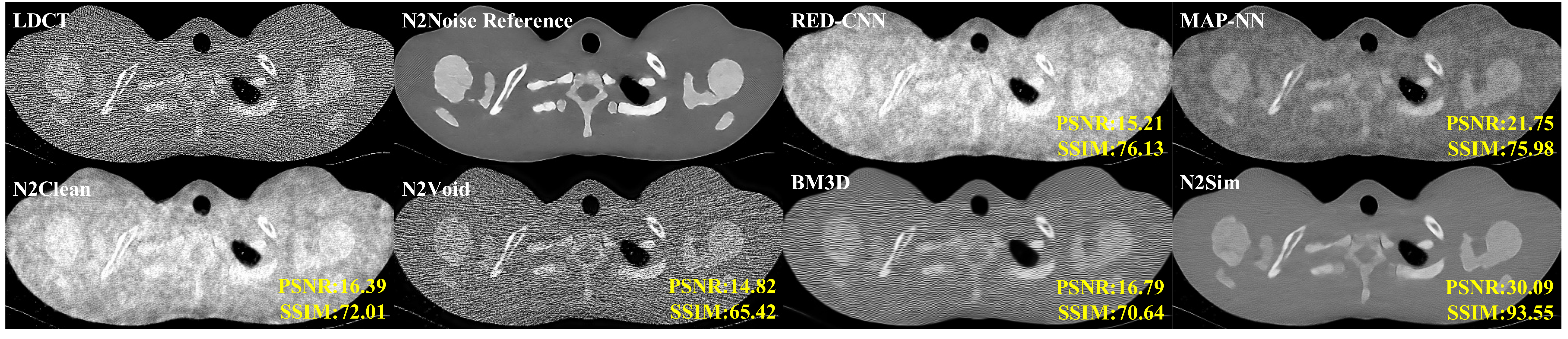}
    \caption{Generalizability results of Noise2Sim on the FDA dataset, significantly better than those from other denoising methods. The LDCT image was obtained with 200mA and \emph{b40f} kernel. The N2Noise reference image was obtained with 200mA and processed with the paired Noise2Noise method. The display window is [-160, 240] HU, along with PSNR and SSIM values.}
    \label{fig:ldctresults-real}
\end{figure}

\noindent \textbf{Generalizability of Noise2Sim on different LDCT datasets.}
We further evaluated Noise2Sim on real LDCT scans of the FDA anthropomorphic phantom \cite{fda}, demonstrating that Noise2Sim performs better than the supervised learning methods that transfer from the simulated to real data.
Since we can hardly obtain paired LDCT and NDCT images in the clinical scenario, we applied the supervised learning models that trained on the Mayo dataset to processing the real LDCT images in the FDA dataset.
For this purpose, Noise2Sim was trained and tested using the same LDCT images in the FDA dataset.
Specifically, we used two low-dose (25mA) phantom scans that were reconstructed with different kernels denoted by \emph{b40f} and \emph{b60f}, resulting in different noise patterns respectively.
To quantify the denoising performance, we used the corresponding normal-dose (200mA) phantom scans as targets. Because the normal-dose images reconstructed with the \emph{b40f} kernel still contain noises, we first processed the normal-dose images by training the Noise2Noise model with the paired noisy images and used the denoised noraml-dose image as the reference.
Table~\ref{tab:fda} shows that Noise2Sim achieves the best results on both two phantoms in terms of PSNR and SSIM among all the supervised and unsupervised learning methods.
Particularly, the supervised learning models cannot work well when being transferred to denoising real LDCT images that are reconstructed with a totally different kernel \emph{b40f}.
The \emph{b40f} phantom images denoised using different methods is visualized in Fig. \ref{fig:ldctresults-real}.
These results show that supervised learning models suffer from severe generalizability issues when there is a shift in distributions between training and testing samples.
In contrast, Noise2Sim can be directly optimized in the target domain without suffering from the generalizability issue.
Surprisingly, Noise2Sim can recover  underlying structures despite strong noises, while these structures can be hardly seen in the original LDCT images.
These strongly demonstrate the effectiveness of the proposed unsupervised deep denoising approach, which has a great potential to improve the LDCT image quality and diagnostic performance.

\begin{figure}[h]
    \centering
    \includegraphics[width=0.99\textwidth]{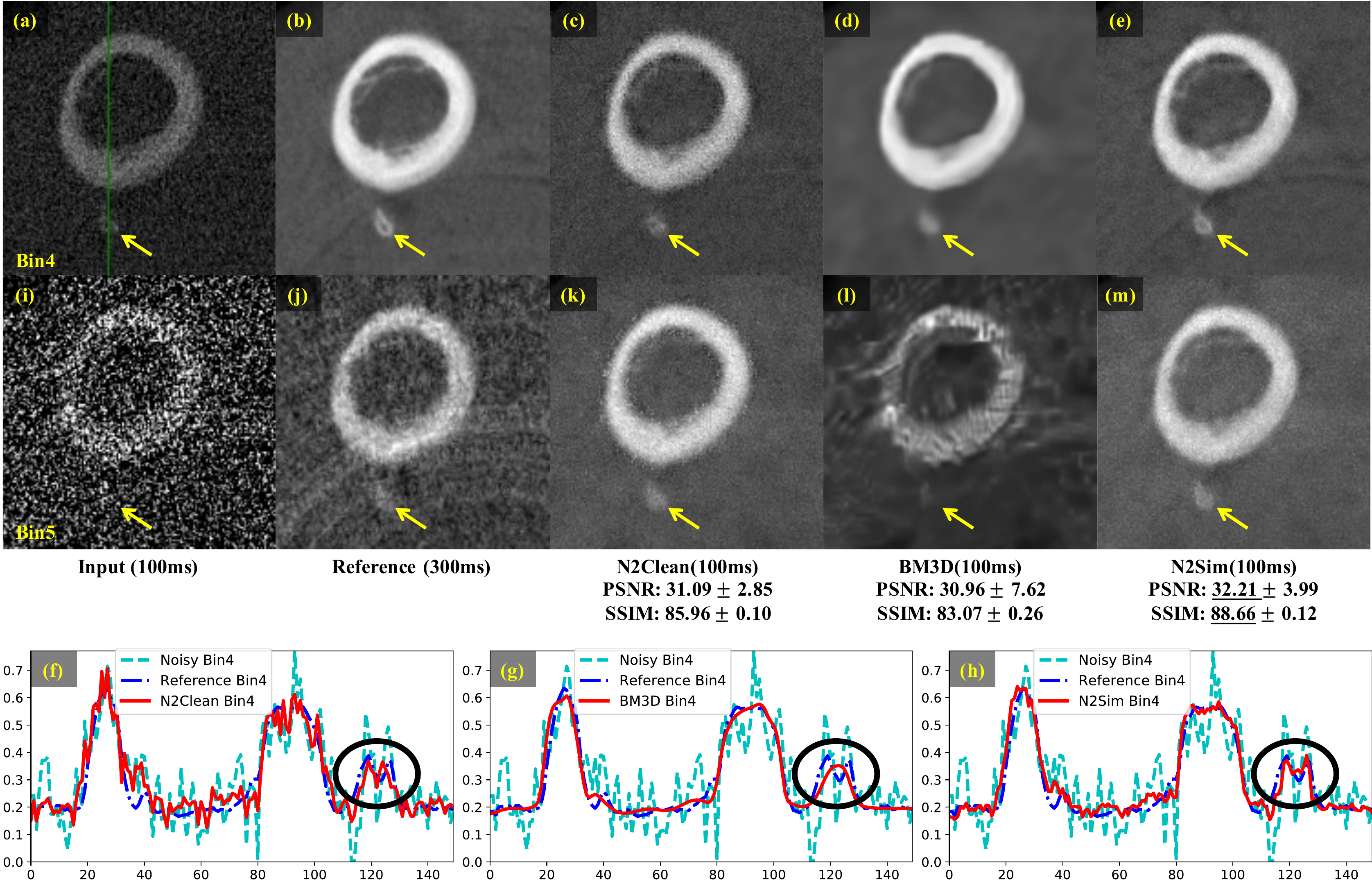}
    \caption{Qualitative and quantitative results of Noise2Clean, BM3D, and Noise2Sim on 4D photon-counting CT images. (a)-(e) and (d)-(m) respectively show denoising results from the $4^{th}$ and $5^{th}$ energy bins. (f)-(h) show the corresponding profiles along the green line in (a). The PSNR and SSIM results over the test dataset are reported for each method including both mean and standard deviation values over the test dataset (40 slices), and the best mean results are underlined. Yellow arrows indicate a small bone structure, with the counterparts in (f)-(h) emphasized within the black circles.}
    \label{fig:pcct-leg}
\end{figure}

\noindent \textbf{PCCT denoising with Noise2Sim recovering features faithfully in multiple channels.}
As an emerging CT imaging technology, an x-ray photon-counting detector records individual x-ray photons as well as their energy levels. By the nature of the photon-counting mechanism, this technique is free of electronic noise and provides a much higher resolution compared to that of conventional energy integrating detectors~\cite{li2020clinical}. The x-ray energy information can be used to extract chemically specific information and facilitate beam hardening correction, metal artifact reduction, tissue characterization and K-edge imaging~\cite{WU2021}. On the other hand, the refinement of spatial resolution and the division of photons into narrow energy bins significantly raise the image noise in individual energy channels, demanding powerful denoising methods that remove image noise while preserving fine features.

To evaluate the effectiveness of Noise2Sim on denoising PCCT images, we scanned a chicken drumstick at low- and normal-dose setting respectively and reconstructed them into PCCT image volumes using the state-of-the-art MARS spectral CT system.
In the experiments, the exposure time was controlled for low-dose (100 ms) and normal-dose (300 ms) images with 5 energy bins being of 7-20, 20-30, 30-47, 47-73 and $>$73keV, respectively.
Although the low-dose and normal-dose PCCT images of the same object are available, it is still hard to directly apply supervised learning methods for PCCT image denoising since different scans were not in perfect registration.
Thus, we first used a registration tool to align the low-dose and normal-dose PCCT volumes so that the supervised leaning method (Noise2Clean) can be applied. Then, the quantitative results were calculated for comparison.
For this purpose, we still selected the unsupervised BM3D method as the baseline.
Fig. \ref{fig:pcct-leg} shows that our proposed unsupervised leaning method outperforms BM3D and even the supervised learning method in terms of both the quantitative and qualitative results.
BM3D requires manually tuning the prior parameter of standard deviation (std) for each bin to achieve decent results. Despite these tedious tuning steps, BM3D still produced either over-smoothed features (the $4^{st}$ bin with the std of 0.2) or structured artifacts (the $5^{th}$ bin with std of 0.5).
The reason seems that BM3D could either blindly remove high-frequency components including some fine structures in the transform domain according to the patch similarity to address a large standard deviation prior or preserve some large structured noises to reflect a small standard deviation prior.
On the other hand, due to the imperfect alignment of low-dose and normal-dose training samples, the performance of a supervised learning method may be degraded.
In contrast, Noise2Sim intrinsically registers similar training samples and can be then applied to produce excellent denoising results without tuning the hyperparameter in a bin-specific fashion.
Particularly, Noise2Sim is able to faithfully recover the underlying structures interfered by strong noises, as shown in the last row of Fig. \ref{fig:pcct-leg}, which is visually even better than the reference image.

\begin{figure}[h]
    \centering
    \includegraphics[width=0.95\textwidth]{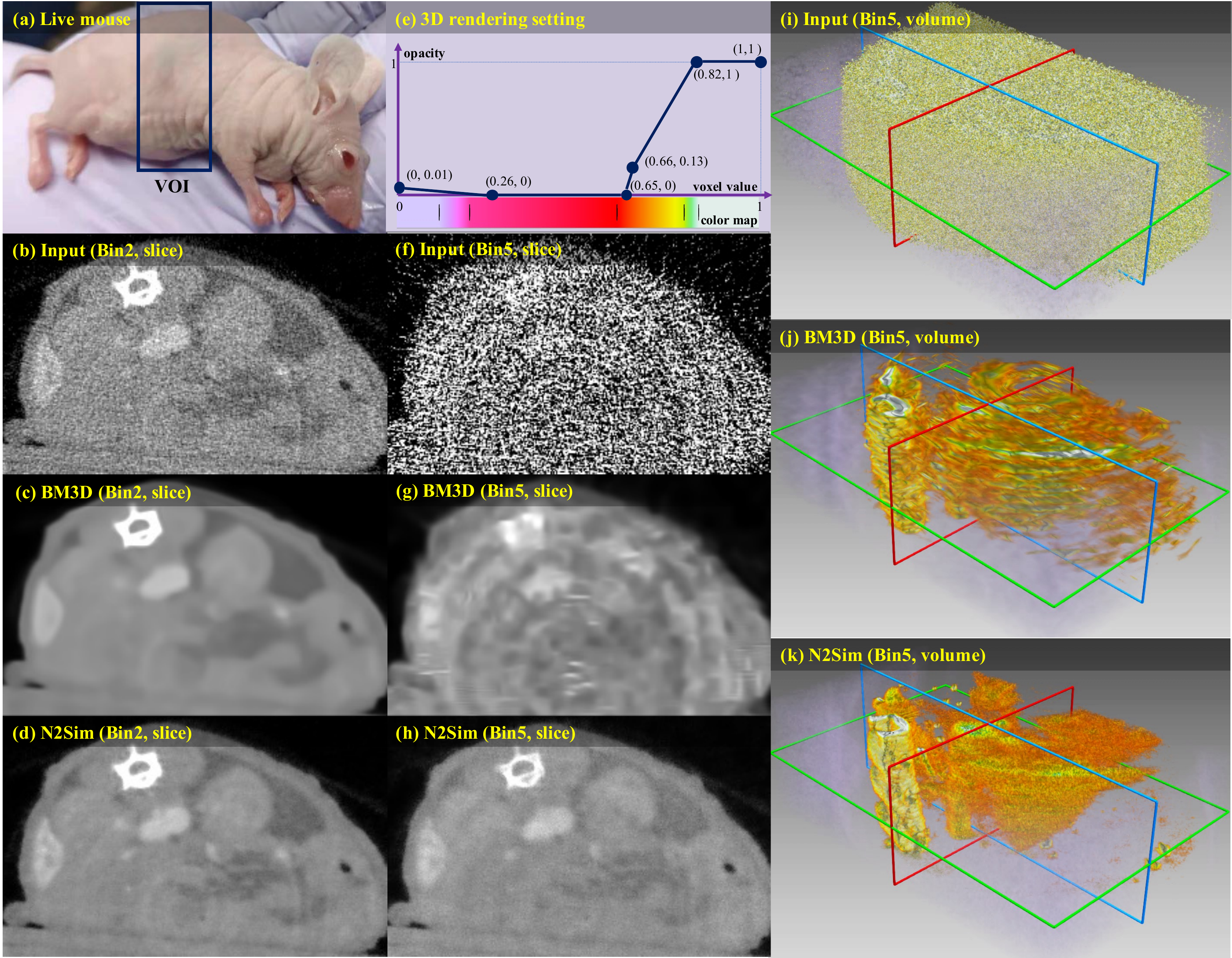}
    \caption{Visualization of the denoised PCCT images of a live mouse.
    (a) shows the live mouse with a boxed area to be scanned.
    (b), (c), and (d) are respectively the input, BM3D and Noise2Sim results of a slice in the $2^{th}$ bin.
    (f), (g), and (h) are the input, BM3D and Noise2Sim results of a slice in the $5^{th}$ bin.
    (i), (j), (k) are the 3D rendered images in the $5^{th}$ bin, where the rendering function of voxel value v.s. opacity and voxel value v.s. color map are given in (e).}
    \label{fig:pcct-mouse}
\end{figure}

Finally, we applied Noise2Sim to a PCCT image volume of a live mouse scanned on the MARS spectral CT system. In this case, only a single normal-dose scan was available so that the supervised leaning method cannot be applied.
Fig. \ref{fig:pcct-mouse} shows this challenging case in the $5^{th}$ bin, where structures can be hardly observed in the original noisy image.
BM3D can only recover a rough contour of the spine.
In contrast, Noise2Sim can accurately recover the spine and other organs.
The details on all the above datasets are described in section S4 the supplementary material, which can be publicly downloaded for reproducible and further research.

\section*{Discussion}

In the Noise2Sim method, the key is to find similarity-based training samples satisfying the ZCN and ZCD conditions (see the Methods section) from noisy data without any labels. Specific to the dimensionality of involved images, several efficient searching algorithms have been evaluated for suppressing noises in this study. Since here we focus on unsupervised learning, the commonly used UNet \cite{unet} architecture and the Euclidean distance were applied for similarity measurement. Clearly, the denoising performance should be better with a more advanced network design, more accurate similarity measurement, and a more sophisticated loss function.

As a general unsupervised denoising approach, Noise2Sim can be adapted to many other domains, not limited to CT images. For example, the utility of Noise2Sim is demonstrated with natural images (see section S1 in the Supplementary Material). In the general spirit of Noise2Sim, deeper analysis of domain-specific data can help fully leverage similar data features and achieve superior performance. As far as LDCT denoising is concerned, the imaging performance may be improved using a dual domain denoising network with similarity matches performed in the sinogram and image domains synergistically. Moreover, due to its simplicity and efficiency, this approach can be incorporated into other frameworks as an element or a constraint to regularize the CT image reconstruction process. 

In conclusion, we have presented a novel similarity-based unsupervised denoising approach only using noisy images with neither clean nor noisy paired labels. To our best knowledge, our proposed denoising approach is the first-of-its-kind to suppress structured noises in the self-learning fashion. Theoretically, we have proved the Noise2Sim theorem that the similarity-based unsupervised method is equivalent to the supervised learning methods under mild conditions (see the Methods section). Also, we have applied the Noise2Sim approach in the important applications to denoise 2D natural images, 3D low-dose CT images and 4D photon-counting micro-CT images. Our experimental results have demonstrated the superiority of the Noise2Sim approach in comparison with existing model-based and deep learning denoising methods. Potentially, the Noise2Sim approach can be adapted to various domains for practical denoising performance.

\section*{Methods}

A noisy image $\x_i$ can be decomposed into two parts: $\x_i = \s_i + \n_i$, which is generated from the joint distribution $p(\s,\n) = p(\s)p(\n|\s)$, where $\s_i$ and $\n_i$ are the clean signal and associated noise respectively. A deep denoising method learns a network function to recover the clean signal $\s_i$ from the noisy signal $\x_i$, i.e., $\y_i = f(\x_i; \bm{\theta})$, where $f$ denotes the network function with a vector of parameters $\bm{\theta}$ to be optimized.
In a supervised training process (Noise2Clean), each noisy image $\bm{x}_i$ is associated with the corresponding clean image $\bm{s}_i$ as the target. 
Let $\bm{\theta}_c$ be the network parameters optimized with paired noise-clean data, we have

\begin{equation}
\label{eq_loss_c}
\bm{\theta}_c = \arg\min_{\bm{\theta}} \frac{1}{N_c} \sum_{i=1}^{N_c} || f(\s_i + \n_i; \bm{\theta}) - \s_i||_2^2,
\end{equation}
where $N_c$ is the number of images, and the mean squared error (MSE) is used as the loss function~\cite{n2n, n2v}.
Among all denoising methods, the supervised deep denoising methods currently achieve the best results but they are handicapped when paired images are unavailable.

The ubiquitous similarity in the physical world leads to the well-known popularity of similar image pixels, patches, slices, volumes and tensors embedded within and across images, which are conveniently referred to as similar sub-images in this study.
Based on this fact, we propose Noise2Sim, a general unsupervised denoising approach for training a deep network without collecting paired noisy or clean target data.
The key idea behind Noise2Sim is to replace explicit clean or noisy targets with the similar targets for training the denoising network, such that noises are suppressed to enhance signals faithfully.
Specifically, given noisy images only, we first construct a set of similar sub-images, denoted by $\bm{x}_i = \bm{s}_i + \bm{n}_i$ and $\hat{\bm{x}}_i = \bm{s}_i + \bm{\delta}_i + \hat{\bm{n}}_i$, where $\bm{\delta_i}$ is the difference between the clean signal components in similar sub-images, and $\bm{n}_i$ and $\hat{\bm{n}}_i$ are two different noise realizations. 
Let $\bm{\theta}_s$ be the vector of network parameters to be optimized with the constructed similar pairs of data. We can find $\bm{\theta}_s$ by minimizing the following loss function: 
\begin{equation}
    \label{eq_loss_s}
    \bm{\theta}_s = \arg\min_{\bm{\theta}} \frac{1}{N_s} \sum_{i=1}^{N_s} || f(\s_i + \n_i; \bm{\theta}) - (\s_i + \bm{\delta}_{i} + \hat{\bm{n}}_i)||_2^2,
\end{equation}
where $N_s$ denotes the number of noisy similar image pairs.
First of all, We present the following theorem to justify our Noise2Sim approach:
\begin{theorem}
\label{theo_1}
Given the zero conditional expectations $\mathbb{E}[\hat{\bm{n}}_{i}|\bm{s}_i + \bm{n}_{i}] = \bm{0}$ and $\mathbb{E}[\bm{\delta}_{i} | \bm{s}_i + \bm{n}_i] = \bm{0}, \forall i$, we have $\lim_{N_s \rightarrow \infty} \bm{\theta}_s = \bm{\theta}_c$.
\end{theorem}
\begin{proof}
Let $\y_i := f(\s_i + \n_i; \bm{\theta})$, expanding the loss function in Eq.~\ref{eq_loss_s} and removing the terms that do not involve $\bm{\theta}$, we have
\begin{equation}
    \begin{aligned}
    \label{eq_ns2}
    \arg\min_{\mathbf{\theta}}& \frac{1}{N_s} \sum_{i=1}^{N_s} ||\y_i - (\s_i + \hat{\bm{n}}_i + \boldsymbol{\delta}_i)||_2^2\\
    =  \arg\min_{\mathbf{\theta}}& \frac{1}{N_s} \sum_{i=1}^{N_s} \left(||\y_i - \s_i||_2^2 -  2{\hat{\bm{n}}_i}^T \y_i -  2\boldsymbol{\delta}_i^T \y_i \right).
    \end{aligned}
\end{equation}
Clearly, the second and third terms respectively become 
\begin{equation}
    \begin{aligned}
& \lim_{N_s \to \infty} \frac{1}{N_s}\sum_{i=1}^{N_s} 2{\hat{\bm{n}}_i}^T \y_i = 2\mathbb{E}[{\hat{\bm{n}}_i}^T \y_i]\\
& \lim_{N_s \to \infty} \frac{1}{N_s}\sum_{i=1}^{N_s} 2\bm{\delta}_i^T \y_i = 2\mathbb{E}[\bm{\delta}_i^T \y_i].
    \end{aligned}
\end{equation}
Because 
\begin{equation}
    \begin{aligned}
& \mathbb{E}[{\hat{\bm{n}}_i}^T \y_i] = \mathbb{E}[\mathbb{E}[\hat{\bm{n}}_i |\y_i]^T\y_i] = \mathbb{E}[\mathbb{E}[\hat{\bm{n}}_i |\s_i + \n_i]^T\y_i] = 0\\
& \mathbb{E}[\bm{\delta}_i^T \y_i] = \mathbb{E}[\mathbb{E}[\bm{\delta}_i|\y_i]^T|\y_i] = \mathbb{E}[\mathbb{E}[\bm{\delta}_i|\x_i + \n_i]^T|\y_i] = 0,
    \end{aligned}
\end{equation}
where the second equivalence is due to the fact that $\y_i = f(\s_i + \n_i; \bm{\theta})$ is deterministic. Hence, the following equations hold true:
\begin{equation}
    \begin{aligned}
& \lim_{N_s\to \infty} \frac{1}{N_s}\sum_{i=1}^{N_s} 2{\hat{\bm{n}}_i}^T \y_i = 0
& \lim_{N_s \to \infty} \frac{1}{N_s}\sum_{i=1}^{N_s} 2\bm{\delta}_i^T \y_i = 0.
    \end{aligned}
\end{equation}
Finally, as $N_s \to \infty$,
\begin{equation}
\begin{aligned}
& \arg\min_{\mathbf{\theta}} \frac{1}{N_s} \sum_{i=1}^{N_s} ||\y_i - (\s_i + \hat{\bm{n}}_i + \bm{\delta}_i)||_2^2 \\
& = \arg\min_{\mathbf{\theta}} \frac{1}{N_c} \sum_{i=1}^{N_c} ||\y_i - \s_i||_2^2.
\end{aligned}
\end{equation}
That is, 
$\lim_{N_s\rightarrow\infty} \bm{\theta}_s = \bm{\theta}_c$
\end{proof}

\noindent \textbf{Zero-mean conditional noise.}
One condition for the Noise2Sim theorem is the zero-mean conditional noise (ZCN), $\mathbb{E}[\bm{n}'_{i}|\bm{s}_i + \bm{n}_{i}] = \bm{0}$.
This condition will be satisfied if similar sub-images have independent and zero-mean noises, i.e., $\mathbb{E}[\bm{n}'_{i}|\bm{s}_i + \bm{n}_{i}] = \mathbb{E}[\bm{n}'_{i}] = \bm{0}$.
As the direct current (DC) offsets in imaging systems are usually calibrated well, the expectation of observations is the real signal, meaning that the noise component has a zero mean.
If noises of all pixels are independent of each other, the ZCN condition is directly satisfied so that the independent noises can be suppressed with Noise2Sim.
In the case of correlated noises, if the distance between two sub-images is greater than the correlation length of noise, their noise components tend to be independent, as demonstrated in the CT applications.
Therefore, by learning between such nonlocal similar sub-images, it is feasible for the denoising network to perform well on correlated noises, based on the Noise2Sim theorem.
It is worth mentioning that there are no specific assumptions on the noise distribution. Hence, Noise2Sim can be adapted to process different noise distributions.

\noindent \textbf{Zero-mean conditional discrepancy.}
The other condition for the Noise2Sim theorem to be valid is the zero-mean conditional discrepancy (ZCD) $\mathbb{E}[\bm{\delta}_{i} | \bm{s}_i + \bm{n}_i] = \bm{0}$. 
In practice, although the ZCD condition cannot be exactly satisfied, a good approximation can be practically achieved when we search intrinsically similar sub-images, as experimentally demonstrated for natural images (see section S3 in Supplementary Material for details).

\noindent \textbf{Noise2Sim v.s. current deep denoising methods.}
Since the ZCN and ZCD conditions are practical, the unsupervised Noise2Sim learning can be regarded as a surrogate of the supervised learning without collecting paired training samples and yet with a wide range of applications.
During training, Noise2Sim searches and constructs similar sub-images globally in the entire image domain, while existing unsupervised methods~\cite{n2v, noise2self} construct training samples using local pixels.
Noise2Sim enjoys superiority over these unsupervised deep denoising methods as their training samples can be regarded as a subset of those for Noise2Sim, as analyzed in the section S2 of Supplementary Material.

\noindent \textbf{Noise2Sim v.s. traditional non-local mean methods.}
The traditional non-local mean (NLM) method independently processes each pixel in each image via weighted averaging non-local similar pixels.
In other words, only its own similar pixels contribute to suppressing noise at the reference pixel in a linear manner.
Noise2Sim is different from NLM in two aspects.
First, Noise2Sim optimizes a non-linear neural network function to suppress image noises nonlinearly, which is more powerful than linear averaging.
Second, Noise2Sim processes every pixels in each image using the pretrained neural network with the same weights, which is the outcome of all similar pixels of all training images.
Thanks to the high capacity of non-linear neural networks, all sets of similar pixels instead of a sing similar set contribute to recovering each specific pixel.

\noindent \textbf{Loss function.} In the Noise2Sim theorem, the equivalency between Noise2Sim and Noise2Clean is demonstrated under the MSE loss function.
The MSE loss function will drive the network to output the mean value of the noise distribution so that two zero-mean conditions are required.
We can also use the L1 loss function to optimize the network, which tends to predict the median value of the noise distribution.
If the median of the noise distribution is closer to zero, L1 loss function will achieve better results, as empirically demonstrated in the LDCT denoising experiments.

\begin{figure}[h]
    \centering
    \includegraphics[width=1.0\textwidth]{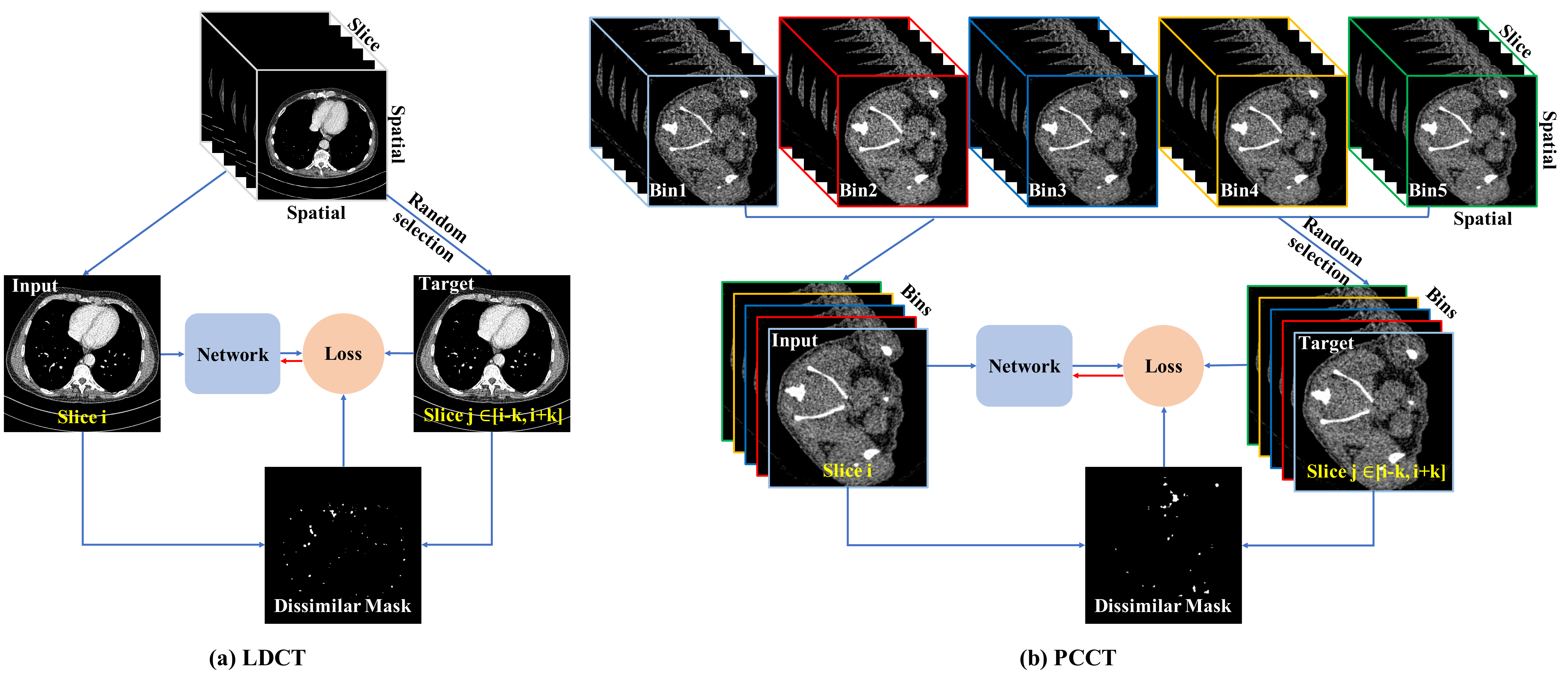}
    \caption{Noise2Sim training process on LDCT and PCCT. The similar volumes along the slice direction are selected to construct training samples, where the dissimilar vectors identified in a mask image are excluded during training.}
    \label{fig:ctsearch}
\end{figure}

\noindent \textbf{Noise2Sim training process on LDCT and PCCT.}
Without paired training samples, we can search for a set of similar sub-images from noisy images with the ZCN and ZCD conditions approximately satisfied.
In this study, similar sub-images refer to similar pixels/patches in 2D images, slices in 3D images, and volumes in 4D images.
Based on the Nosie2Sim theorem, the denoising network can be optimized with the constructed similar training samples in a self-learning manner asymptotically up to the supervised learning performance.
Generally, a similar training set can be defined as $\{(\bm{x}_i, \hat{\bm{x}}_i) | S(T(\bm{x}_i), T(\hat{\bm{x}}_i)) \}$, where $T$ is a transform of a sub-image, and $S$ is a metric to identify if two sub-images are similar or not. In practice, $T$ and $S$ may take different forms, depending on domain-specific priors.

It is well known that CT noise is spatially structured, which varies with the scanning protocol parameters, reconstruction steps, patient conditions, scanner hardware factors, etc.
Geometrically, each detector row mainly contributes to one or several consecutive slices, and different detector rows are subject to independent noise realizations.
As a result, there is a much weaker noise correlation between different slices (especially when the slices are well separated) than within the same slice.
On the other hand, the same organ or tissue in a patient usually has the similar Hounsfield unit (HU) values.
With these priors, similar slices with approximately independent and zero-mean noises can be obtained along the longitudinal axis of the patient.
As for a volumetric CT image, it consists of two in-plane dimensions and one through-slice dimension, as shown in Fig.~\ref{fig:ctsearch} (a).
Given the $i^{th}$ reference slice as input, a similar slice can be randomly selected from the $(i-k)^{th}$ to $(i+k)^{th}$ slices, where $k$ defines the searching range of similar slices, and used as the target during training.
Similarly, a PCCT image tensor is of four dimensions, including three spatial dimensions (in-plane and through slice), and a channel dimension, as shown in Fig.~\ref{fig:ctsearch} (b).
Thus, the $i^{th}$ slice for PCCT is a 3D tensor (two spatial and one spectral), and a similar tensor is also randomly selected from $(i-k)^{th}$ to $(i+k)^{th}$ slices.
However, it cannot be guaranteed that the pixels/vectors at the same location in different neighboring slices are always similar to each other, especially when the structures are longitudinally changed.
According to the Noise2Sim theorem, these dissimilar parts will compromise the zero-mean conditional discrepancy condition, and thus should be excluded from training samples.
Here propose to identify dissimilar pixels between similar slices, which are indicated by the dissimilar mask in Fig.~\ref{fig:ctsearch} and excluded in computing the loss.

Specifically, we denote a pair of similar LDCT or PCCT images as $\bm{x}_i, \bm{x}_j \in R^{H \times W \times C}$, where $H, W, C$ denote height, width, and channel of CT images, $C=1$ for LDCT images, and $i, j$ are the slice indices, $j \in [i-k, i+k]$. For each pair of vectors $\bm{x}_i(u, v, :), \bm{x}_j(u, v, :) \in R^C$ at the same spatial location $(u, v)$, we use their associated patches to determine their similarity. The patches of these two vectors share the same spatial coordinates $S(u, v)$ with the patch size $s \times s$. Formally, we define the distance map $\bm{d} \in R^{H \times W}$ between $\bm{x}_i$ and $\bm{x}_j$ as
\begin{equation}
    \bm{d}_{ij}(u, v) = \frac{1}{C}\sum_{c=1}^C\sqrt{ \left(\frac{1}{s^2}\sum_{(p, q) \in S(u, v)} (\bm{x}_i(p, q, c) - \bm{x}_j(p, q, c)) \right)^2}.
\end{equation}
In practice, the inner summation can be computed by the convolution with the $s \times s$ kernel whose elements are ones.
Then, the dissimilar mask $\bm{m}_{ij}$ is computed as
\begin{equation}
  \bm{m}_{ij}(u, v) =
    \begin{cases}
      1 & \bm{d}_{ij}(u, v) > d_{th}\\
      0 & \text{otherwise}
    \end{cases}
\end{equation}
where $d_{th}$ is a predefined threshold. In all experiments, we empirically set the patch size $s=7$ and the threshold $d_{th}=30$ in HU. Finally, the loss function is
\begin{equation}
    L = \frac{1}{N_s}\sum_{i, j}||(f(\x_i; \bm{\theta}) - \x_j) \odot \bm{m}_{ij}||_2^2,
\end{equation}
where $\odot$ denotes the spatial-wise multiplication, and MSE function can be replaced by $L_1$ function. More training details can be found in section S6 of the supplementary material.

\subsection*{Data Availability}
The low-dose CT datasets from Mayo and FDA used in this study can be respectively obtained through \url{www.aapm.org/grandchallenge/lowdosect/} and \url{https://wiki.cancerimagingarchive.net/display/Public/Phantom+FDA}.
The natural image datasets including BSD68, BSD500, and Kodak used in section S1 of the Supplementary Material can be respectively obtained through \url{www.github.com/clausmichele/CBSD68-dataset}, \url{www.eecs.berkeley.edu/Research/Projects/CS/vision/bsds/}, and \url{www.cs.albany.edu/~xypan/research/snr/Kodak.html}.
All photon-counting datasets scanned in this study are publicly available on our project page \url{http://chuangniu.info/projects/noise2im/}.

\subsection*{Code Availability}
Our source codes for Noise2Sim are publicly available at \url{https://github.com/niuchuangnn/noise2sim}. Noise2Sim can be installed through Pip and source.



\section*{Acknowledgments}
\textbf{Funding:} This work was supported in part by NIH/NCI under Award numbers R01CA233888, R01CA237267, R21CA264772, and NIH/NIBIB under Award numbers R01EB026646, R01HL 151561, R01EB031102.
\textbf{Author contributions:} G. Wang and C. Niu originated the idea. C. Niu designed specific algorithms and conducted all experiments. C. Niu and F. Fan established the mathematical analysis. M. Li, X. Guo, and W. Wu collected and preprocessed the photon-counting data. M. Li analyzed properties of CT noises. C. Niu and G. Wang drafted the manuscript, all co-authors participated discussion, contributed technical points, and revised the manuscript iteratively.
\textbf{Competing interests:} All authors declare that they have no competing interests.
\textbf{Data and materials availability:} All data used in the paper are publicly available at the project page \url{http://chuangniu.info/projects/noise2im/}.

\section*{Supplementary Material}

The Supplementary Material includes:\\
Section S1: Application of Noise2Sim to Natural images.\\
Section S2: Analysis Among Noise2Noise, Noise2Void, and Noise2Sim.\\
Section S3:  Estimation of Conditional Discrepancy.\\
Section S4: Dataset Details.\\
Section S5: Network Architectures.\\
Section S6: Training Details.\\
Section S7: More Denoising Results on PCCT.\\
References \textit{(38-41)}

\renewcommand{\thepage}{S\arabic{page}} 
\renewcommand{\thesection}{S\arabic{section}}  
\renewcommand{\thetable}{S\arabic{table}}  
\renewcommand{\thefigure}{S\arabic{figure}}

\section{Application of Noise2Sim to Natural images}

Natural image benchmark datasets are commonly used to evaluate the effectiveness of new denoising algorithms.
In this section, we systemically evaluate the Noise2Sim theorem on the commonly used natural image datasets, including BSD68 \cite{7839189} containing grayscale images, BSD500 \cite{amfm_pami2011} and Kodak (http://r0k.us/graphics/kodak/) containing color images.
In the following subsections, we introduce Constructing Training Samples on 2D Images with Independent Noise, Comparison Results, and Ablation Studies, respectively.

\subsection{Noise2Sim Training on 2D Images with Independent Noise}

\label{sec_nisearch}
\begin{figure*}[h]
    \centering
    \includegraphics[width=1\textwidth]{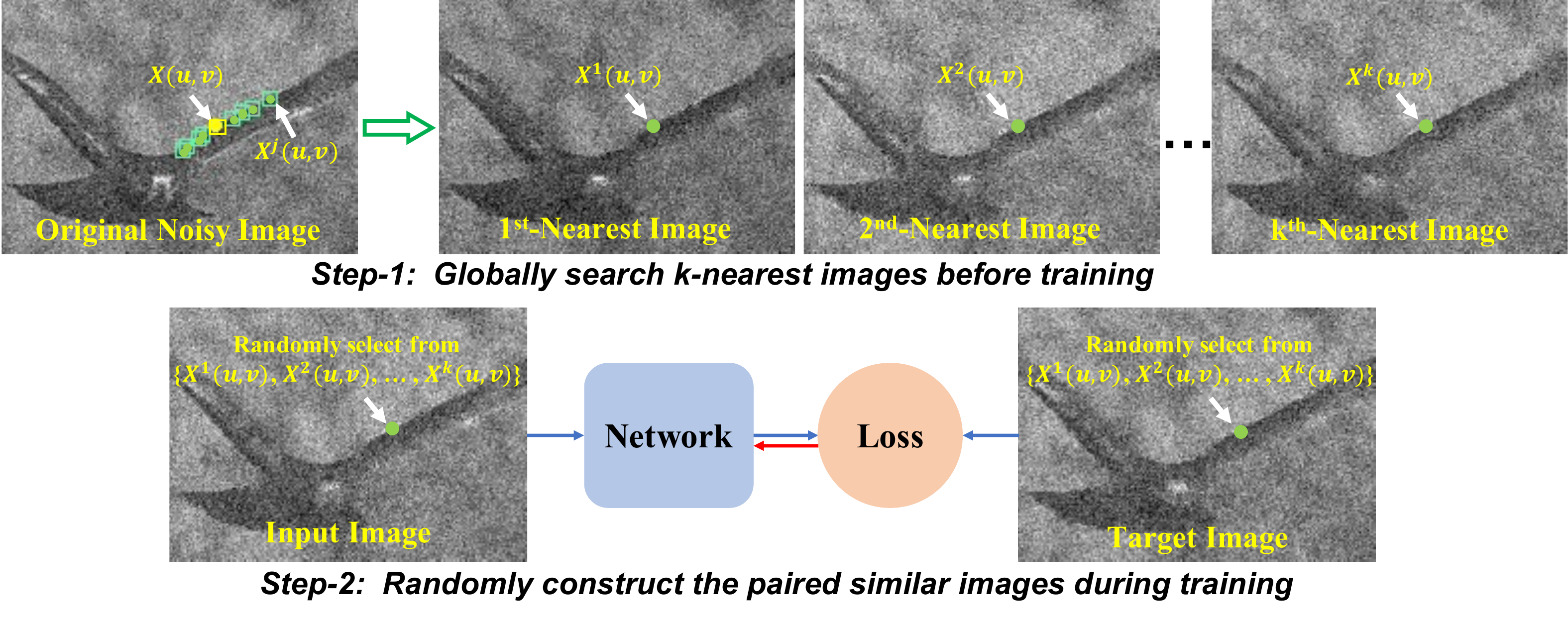}
    \caption{Noise2Sim training process on 2D images with independent noise. Step 1 is to search for a set of $k$ similar pixels for each pixel in the original noisy image, and form $k$ most similar images, which is also referred to as nearest images. In the original noisy image, the yellow and green points respectively denote the reference pixel $\bm{X}(u,v)$ and its $k$ nearest pixels $\bm{X}^j(u, v), j=1,\cdots, k$,  and the corresponding boxes present the patches respectively centered at these involved pixels that are used to compute the similarity of center pixels. The $k^{th}$ nearest image is formed by replacing all pixels in the original noisy image with their $k^{th}$ nearest pixels. Step 2 is to construct a pair of similar images as the input and the target to train a deep neural network. Each similar image is independently constructed by replacing every pixel $\bm{X}(u, v)$ in the original noisy image with the pixel randomly selected from $\mathcal{N}(u,v) = \{\bm{X}(u,v), \bm{X}^1(u,v), \cdots, \bm{X}^k(u,v)\}$, which are pre-computed in Step 1.}
    \label{fig:nisearch}
\end{figure*}

In the case of 2D images corrupted by independent noise, we propose an efficient two-step algorithm to construct a large set of similar images from noisy images only.
As introduced in the main text, the similar training set is defined as $\{(\bm{x}_i, \hat{\bm{x}}_i) | S(T(\bm{x}_i), T(\hat{\bm{x}}_i)) \}$, where $T$ is the transformation function for sub-images, and $S$ is the function to identify if two sub-images are similar or not.
Here the sub-images $x_i$ are defined as pixels, and similar images are constructed by replacing original pixels with the searched similar pixels during training. 
For simplicity, here $T$ is the identity function that means no transformations are applied to image pixels, and one can also use some transformations to reduce the variance of similarity estimation caused by noises \cite{bm3d}.
For the similarity estimation $S$, we adopt the k-NN strategy that each pixel is matched with k nearest similar pixels in terms of the Euclidean distance between their surrounding patches.

Specifically, let us describe this process at the pixel-level.
For each reference pixel $\bm{x}(u,v)$ with its coordinates $(u,v)$ in a given noisy image $\bm{x}$, we compute its $k$ nearest pixels over the whole image. The distance between two pixels $\bm{x}(u_1, v_1)$ and $\bm{x}(u_2, v_2)$ is defined as the Euclidean distance between their associated patches; i.e., $||\bm{S}(u_1, v_1) - \bm{S}(u_2, v_2)||_2$, where $\bm{S}(u, v)$ denotes a square patch that is determined by the pre-defined patch size and the center pixel $\bm{x}(u,v)$. Thus, each position in the image has a set of $k+1$ similar pixels (+1 means the reference pixel included), denoted as $\mathcal{N}(u, v) = \{\bm{x}(u,v), \bm{x}^1(u,v), \cdots, \bm{x}^k(u,v)\}$, where $\bm{x}^j(u,v)$ denotes the $j$-th nearest pixel relative to $\bm{x}(u,v)$. As shown in Fig. \ref{fig:nisearch}, the yellow and green dots denote the reference pixel and its nearest pixels, and the boxes are the associated patches.
Based on these similar pixel sets, a similar noisy image can be constructed by replacing every original pixel $\bm{x}(u, v)$ with a similar one randomly selected from $\mathcal{N}(u, v)$. Then, a pair of similar images is independently constructed in each iteration.

The number of all possible similar images to each given image is $(k+1)^{H \times W}$, where $H$ and $W$ represent the image height and width respectively.
If all similar images are prepared before training or on-the-fly during training, the memory space or computational time will be unacceptable.
Naturally, we propose to divide the Noise2Sim training process into two steps. First, we generate $k$-nearest similar images from a single noisy image. The $k$-nearest similar images are obtained by sorting the $k$-nearest similar pixels for each pixel location, i.e., the $j$-th nearest image is $[\bm{x}^j(u, v)]_{H\times W}$, as shown in Fig. \ref{fig:nisearch}.
Second, with these $k+1$ similar images, we randomly and independently construct a pair of similar images on-the-fly during training.
The  time searching for these similar images is acceptable in the first step, using an advanced  algorithm  on GPU. The construction of paired similar images takes little time in the second step.
There are alternative ways to construct the training pairs, see SI C-B for comparative results.
Also, since noise may harm the estimation of signal similarities, we propose to use the denoised image to improve the computation of the similarity between image patches, and then conduct Noise2Sim training again, which can be iterated if necessary; please see Subsection S5-\ref{sec_iter} for details.
In Subsection S5-\ref{sec_para}, we comprehensively studied the effects of the involved hyperparameters, the patch size and the neighborhood rand $k$.

\subsection{Comparison Results}

\begin{table}[htp]
\scriptsize
 \caption{PSNR results obtained using different denoising methods on the BSD68 dataset corrupted by Gaussian noise at different levels. Std means the standard deviation of Gaussian noise.}
  \centering
  \begin{tabular}{ccccccc}

   Std  & 15 & 25 & 35 & 45 & 55 & 65  \\
    \midrule
   $N2Clean$  & 31.62   & 28.96 &  27.33    & 26.17     & 25.33     & 24.58        \\   

   $N2Noise$ & 31.39   & 28.88 &  27.22    & 26.01     & 25.28     & 24.42        \\
   \hline
   $N2Void$   & 29.26    & 27.72      & 26.61     & 25.36     & 24.84    & 23.95        \\

   $NLM$  & 29.92    & 27.58      & 26.07     & 24.95     & 24.06     & 23.29      \\

   $N2Sim$    & \textbf{30.25}  & \textbf{28.27} & \textbf{26.98} & \textbf{25.75} & \textbf{24.89} & \textbf{24.06} \\
    \bottomrule
  \end{tabular}
  \label{tab_ns}
\end{table}

\begin{figure*}
    \centering
    \includegraphics[width=1\textwidth]{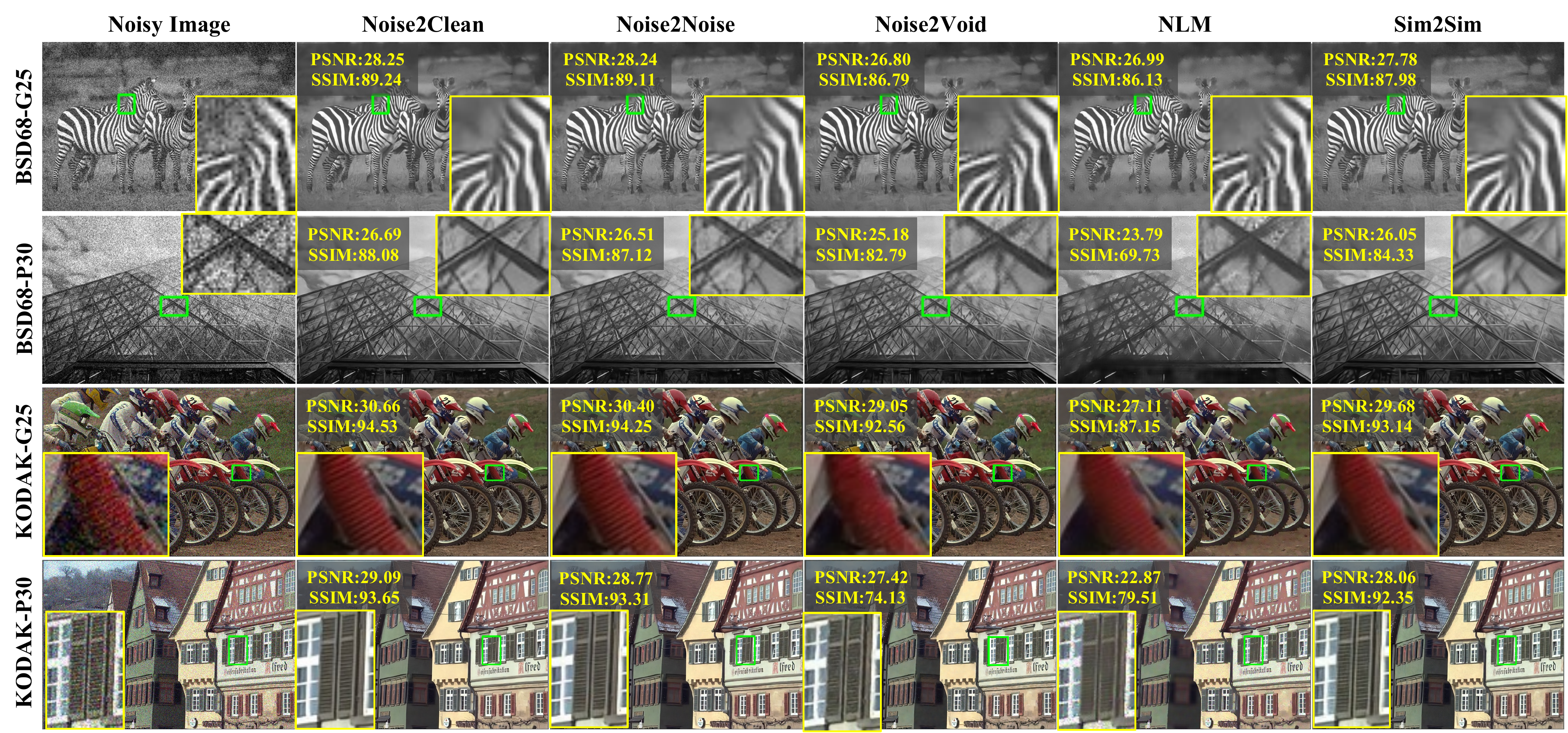}
    \caption{Visual comparison of denoising results on different datasets, where \emph{G25} means Gaussian noise with $Std=25$, and \emph{P30} means Poisson noise with $\lambda=30$. The PSNR and SSIM values are included.}
    \label{fig_nisvisresults}
\end{figure*}

\begin{table}[htp]
\scriptsize
  \renewcommand{\arraystretch}{1.5}
  \renewcommand\tabcolsep{3pt}
 \caption{Comparison of denoising results using different methods on representative datasets, where PSNR and SSIM(\%) values were reported.}
  \centering
  \begin{tabular}{lrrrrr}

   Dataset & Noise2Clean & Noise2Noise & Noise2Void & Noise2Sim \\
     \midrule
    BSD68-G25 & 28.96/89.78  & 28.88/89.44  & 27.72/86.65  & 28.27/88.17  \\
    BSD68-P30 & 28.08/88.26  & 27.98/87.78  & 27.08/85.24  & 27.37/85.70  \\
    Koak-G25 & 32.45/94.23  & 31.68/93.75  & 31.02/92.48  & 31.18/92.52  \\
    Koak-P30 & 31.74/93.65  & 31.52/93.41  & 30.35/91.34  & 30.44/91.56  \\
    \bottomrule
  \end{tabular}
  \label{tab:bsd68-compare}
\end{table}

\begin{table}[htp]
\scriptsize
  \renewcommand{\arraystretch}{1.5}
  \renewcommand\tabcolsep{1.3pt}
 \caption{Comparison results in terms of PSNR, SSIM, and runtime obtained using different methods on BSD68 with additive Gaussian noise of Std 25.}
  \centering
  \begin{tabular}{lrrrrrrr}

   Method & BM3D  & NLM & SNT & N2Void & N2Self & N2Same & N2Sim \\
    \midrule
   PSNR &  28.59 & 27.58  & 27.22 &  27.72    & 27.07     & 28.02     & 28.27    \\
   SSIM(\%) &  88.26 & 85.27 & 83.06  &  86.65    & 87.36        & 87.18    & 88.17    \\
   Runtime(s) &  3.32 & 0.43   & 141.87 &  0.001  & 0.002     & 0.002     & 0.001    \\
    \bottomrule
    
  \end{tabular}
  \label{tab_more}
\end{table}

For a fair comparison, we used the same datasets and network architecture as those used in the Noise2Void work~\cite{n2v}, and independent Gaussian and Poisson noises were used.
Let us first compared Noise2Sim with the baseline methods including Noise2Clean, Noise2Noise~\cite{n2n}, Noise2Void~\cite{n2v}, and NLM~\cite{nlm} methods in suppressing noises at different levels on the commonly used BSD68 dataset~\cite{7839189}, as shown in Table \ref{tab_ns}.
These results show that Noise2Sim is consistently better for a large range of noise levels $15<=std<=65$ than the competing unsupervised methods.
Then, we compared Noise2Sim with the baseline models on more datasets including the BSD500~\cite{amfm_pami2011} and Kodak datasets, which were corrupted by either Gaussian or Poisson noises.
The quantitative and visual results are shown in Table~\ref{tab:bsd68-compare} and Fig.~\ref{fig_nisvisresults}, demonstrating that Noise2Sim is consistently better than Noise2Void and NLM.
These results support that the Noise2Sim method can effectively remove additive Gaussian noises and conditional Poisson noises in gray-scale and color images respectively.
Moreover, we compared Noise2Sim with closely relevant methods in terms of PSNR and inference time, including BM3D~\cite{bm3d}, Noise2Self~\cite{noise2self}, Noise2Same~\cite{noise2same}, and SNT~\cite{snt} methods, as reported in Table~\ref{tab_more}.
Here the standard deviation (Std) of Gaussian noise was set to 25 (the image value range is [0, 255]), and the parameter of Poisson noise to $\lambda = 30$.
Noise2Sim has the same inference time as Noise2Clean, Noise2Noise, and Noise2Void, directly taking the noisy image as input and outputting the denoised image without any additional processing.
Noise2Self and Noise2Same take a little longer inference time as their networks are deeper than Noise2Void and Noise2Sim.
During training, Noise2Sim requires extra time to pre-compute similar images as compared with Noise2Clean and Noise2Noise that assume target images availabe in the first place. Nevertheless, this pre-processing time is very short (about 18 seconds for BSD400 dataset) relative to the network training time (several hours).
Although the classic BM3D method achieves the results comparable to the Noise2Sim counterparts, the inference time of BM3D is thousands times of that of Noise2Sim, which is consistent with what was reported in the Noise2Void paper~\cite{n2v}. Therefore, Noise2Sim is both effective and efficient in general.

\subsection{Ablation Studies}

\subsubsection{Effects of Patch Size and Neighborhood Range on Denoising Performance}
\label{sec_para}

\begin{table}[htp]
\scriptsize
  \renewcommand{\arraystretch}{1.5}
  \renewcommand\tabcolsep{4.7pt}
 \caption{PSNR results affected by patch size and noise level.}
  \centering
  \begin{tabular}{lrrrrrrrr}

    \diagbox[width=5em,height=2.2em]{s}{Std}  & 5 & 15 & 25 & 35 & 45 & 55 & 65 & 75 \\
    \hline
    $3 \times 3 $  & 33.85  & 29.90  & \textbf{28.14} & \textbf{26.73} & \textbf{25.73} & 24.89 & 24.06 & 23.45 \\

    $5 \times 5$  & \textbf{34.32}  & \textbf{29.98}  & 28.02 & 26.72 & 25.69 & 24.93 & 24.03 & 23.45 \\

    $7\times7$  & 33.13  & 29.87  & 27.93 & 26.70 & 25.67 & 24.83 & 24.02 & \textbf{23.74} \\

    $11\times 11$  & 32.10  & 29.82  & 27.43 & 26.56 & 25.69 & 24.86 & \textbf{24.25} & 23.71 \\

    $15 \times 15$  & 31.97  & 28.96  & 27.58 & 26.53 & 25.67 & \textbf{24.94} & 24.18 & 23.68 \\
    \hline
  \end{tabular}
  \label{tab_pn}
\end{table}

In the Noise2Sim method, searching for an appropriate set of similar pixels is a core task at each reference pixel.
In the search, a size-fixed square patch window is translated over a noisy image to find similar pixels. Hence, the patch size, denoted by $s$, is a key parameter that affects the accuracy of similarity estimation.
The effect of different patch sizes vs. various  Gaussian noise levels on PSNR are shown in Table \ref{tab_pn}. It can be seen  that the denoising performance of smaller patch sizes is better for lower noise levels but that of larger patch sizes is better for higher noise levels. This is due to the fact that it requires more contextual information to estimate the similarity accurately when pixels are heavily corrupted by noise.

\begin{table}[htp]
\scriptsize
  \renewcommand{\arraystretch}{1.5}
  \renewcommand\tabcolsep{8.3pt}
 \caption{PSNR results affected by the number of similar pixels and the patch size, for Gaussian noise of $Std=25$.}
  \centering
  \begin{tabular}{lrrrrrr}

   \diagbox[width=6em,height=2em]{s}{k}  & 2 & 4 & 8 & 16 & 32 & 64 \\
     \hline
    $3 \times 3$  & 26.93  & 27.90  & \textbf{28.15} & 28.08 & 27.95 & 27.01  \\

    $5 \times 5$  & 26.95  & 27.84  & \textbf{28.02} & 26.24 & 26.55 & 25.94  \\

    $7 \times 7$  & 26.91  & 27.68  & \textbf{27.93} & 27.71 & 27.50 & 27.28  \\

    $11 \times 11$  & 26.81  & \textbf{27.58}  & 27.43 & 27.23 & 25.21 & 25.02  \\

    $15 \times 15$  & 26.78  & 27.15  & \textbf{27.58} & 26.57 & 26.47 & 24.76  \\
    \hline
  \end{tabular}
  \label{tab_ks}
\end{table}

Moreover, the number of selected similar pixels $k$ determines the error term $\bm{\delta}$ defined in the \textbf{Methods} Section in the main text. 
Then, we evaluated the effect of the number of selected similar pixels on the denoising performance as shown in Table \ref{tab_ks}. The results show that the best denoising performance was generally associated with $k=8$ in our study.
These results suggest that there is a trade-off between the  error term and the number of training samples. Specifically, increasing the number of similar pixels will increase $\bm{\delta}$ values, while decreasing this number will increasing the noise residual in the denoised image. A good balance is $k=8$ in our experiments.

\subsubsection{Effects of Iterative Training Cycles on Denoising Performance}
\label{sec_iter}

\begin{figure}[h]
    \centering
    \includegraphics[width=0.7\textwidth]{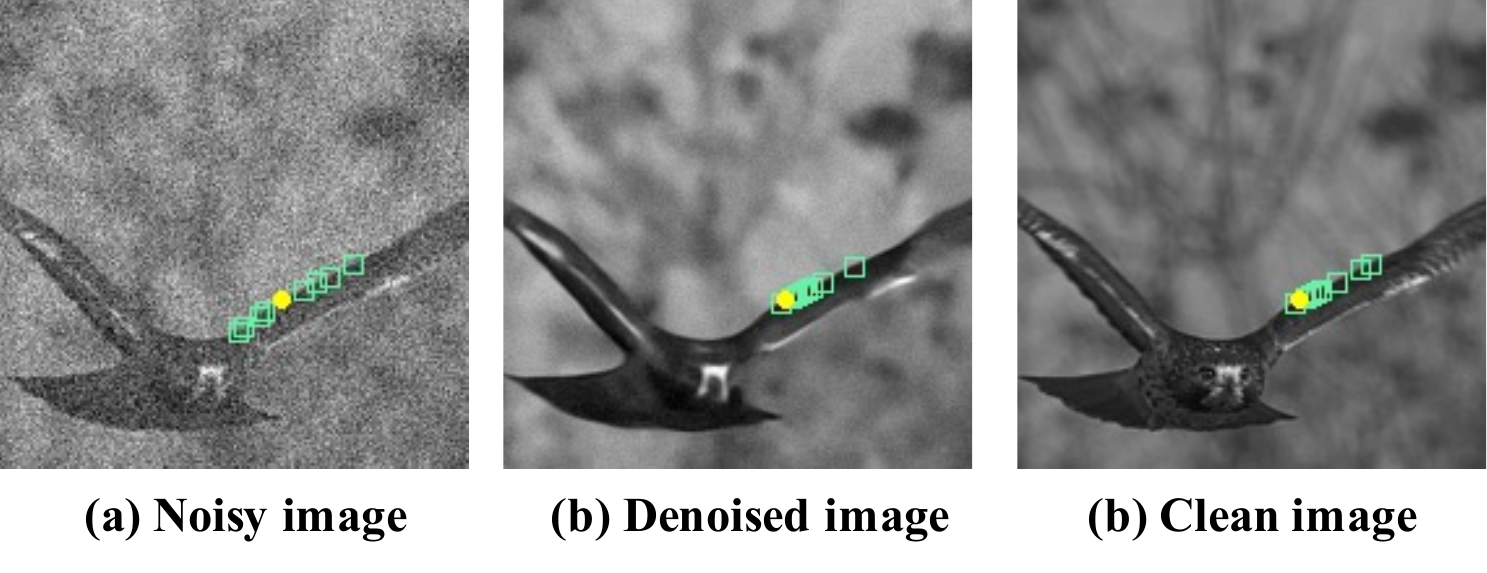}
    \caption{Distribution of similar patches in (a) an original noisy image, (b) the denoised image, and (c) the clean image.}
    \label{fig:cascade}
\end{figure}

\begin{table}[htp]
\scriptsize
  \renewcommand{\arraystretch}{1.5}
  \renewcommand\tabcolsep{6.5pt}
 \caption{PSNR results after iterative training at different noise levels (Std from 5 to 65).}
  \centering
  \begin{tabular}{lrrrrrrr}

    Std  & 5 & 15 & 25 & 35 & 45 & 55 & 65 \\
    \hline
    $W/O$  & 33.85  & 29.90  & 28.14 & 26.73 & 25.73 & 24.89 & 24.06 \\

    $With$  & 34.15  & 30.25  & 28.27 & 26.98 & 25.75 & 24.88 & 24.02 \\

    $Upper$  & 34.22  & 30.54  & 28.60 & 27.14 & 26.08 & 25.15 & 24.33 \\
    \hline
  \end{tabular}
  \label{tab:refine}
\end{table}

Any estimate of similarity between image patches is necessarily compromised in a noisy image, such an estimate may be improved in a denoised image produced by a trained denoising model. That is, the Noise2Sim idea can be repeatedly applied to refine the resultant denoising model. By doing so, the similarity measures can be gradually improved, leading to a superior denoising performance. Fig. \ref{fig:cascade} shows the change in the distribution of similar image patches after one iteration, and the distribution in the denoised image is very close to that in the clean image.
We also evaluated the iterative training by computing the similarity measures in reference to the clean image.
Table \ref{tab:refine} shows that the iterative training enhances the denoising performance, especially when noise level is small. When the images are severely corrupted, the initial denoising results are not good enough. As a result, iterative training might compromise the estimation of similar pixels.

\subsubsection{Effects of Pairing Methods on Denoising Performance}
\label{sec:similarity-mapping}

\begin{table}[htp]
\scriptsize
  \renewcommand{\arraystretch}{1.5}
  \renewcommand\tabcolsep{13pt}
 \caption{Results obtained using different methods for pairing similar pixels in patches of $3 \times 3$ for $k=8$ and Gaussian noise Std 25.}
  \centering
  \begin{tabular}{lrrrr}
   Pairing Method  & 1 & 2 & 3 & 4 \\
     \hline
    PSNR & 27.12  & 27.78  & 27.48 & \textbf{28.15}   \\
    \hline
  \end{tabular}
  \label{tab:map}
\end{table}

Given a set of $k+1$ similar images as described in Subsection S1-\ref{sec_nisearch}, there are four reasonable methods for pairing them: 1) Pair the original noisy image as the input to its randomly constructed similar image as the target; 2) reverse the input and label used in 1); 3) pair two of $k$ sorted similar images without pixel-wise randomization; and 4) pair similar images that were randomly and independently constructed pixel-wise.
The results in Table \ref{tab:map} demonstrate that the fourth pairing method achieves the best denoising performance.
The reason seems that, in the fourth pairing implementation, we use two randomly reconstructed images, which are $\s_i + \bm{\delta}'_i + \n_i$ and $\s_i + \bm{\delta}_i + \hat{\bm{n}}_i$, to train a denoising network, instead of using $\s_i + \n_i$ and $\s_i + \bm{\delta}_i + \hat{\bm{n}}_i$ as stated in the Noise2Sim theorem. Actually, they are equivalent as both of them are paired similar images. Let $\hat{\s}_i = \s_i + \bm{\delta}'_i$, then the randomly reconstructed pairs can be rewritten as $\hat{\s}_i + \n_i$ and $\hat{\s}_i + \bm{\delta}''_i + \hat{\bm{n}}_i$, where $\bm{\delta}''_i = \bm{\delta}_i - \bm{\delta}'$. We can verify that $E[\hat{\bm{n}}_i|\hat{\s}_i + \n_i] = \bm{0}$ and  $E[\bm{\delta}''_i|\hat{\s}_i + \n_i] = \bm{0}$, and the Noise2Sim Theorem remains true. The advantage of such an implementation is that the capacity of possible image pairs can be significantly increased from $(k+1)^{HW}$ to $(k+1)^{2HW}$, leading to a better performance.

\section{Analysis Among Noise2Noise, Noise2Void, and Noise2Sim}
\label{app_morecompare}

\begin{figure}[h]
    \centering
    \includegraphics[width=0.7\textwidth]{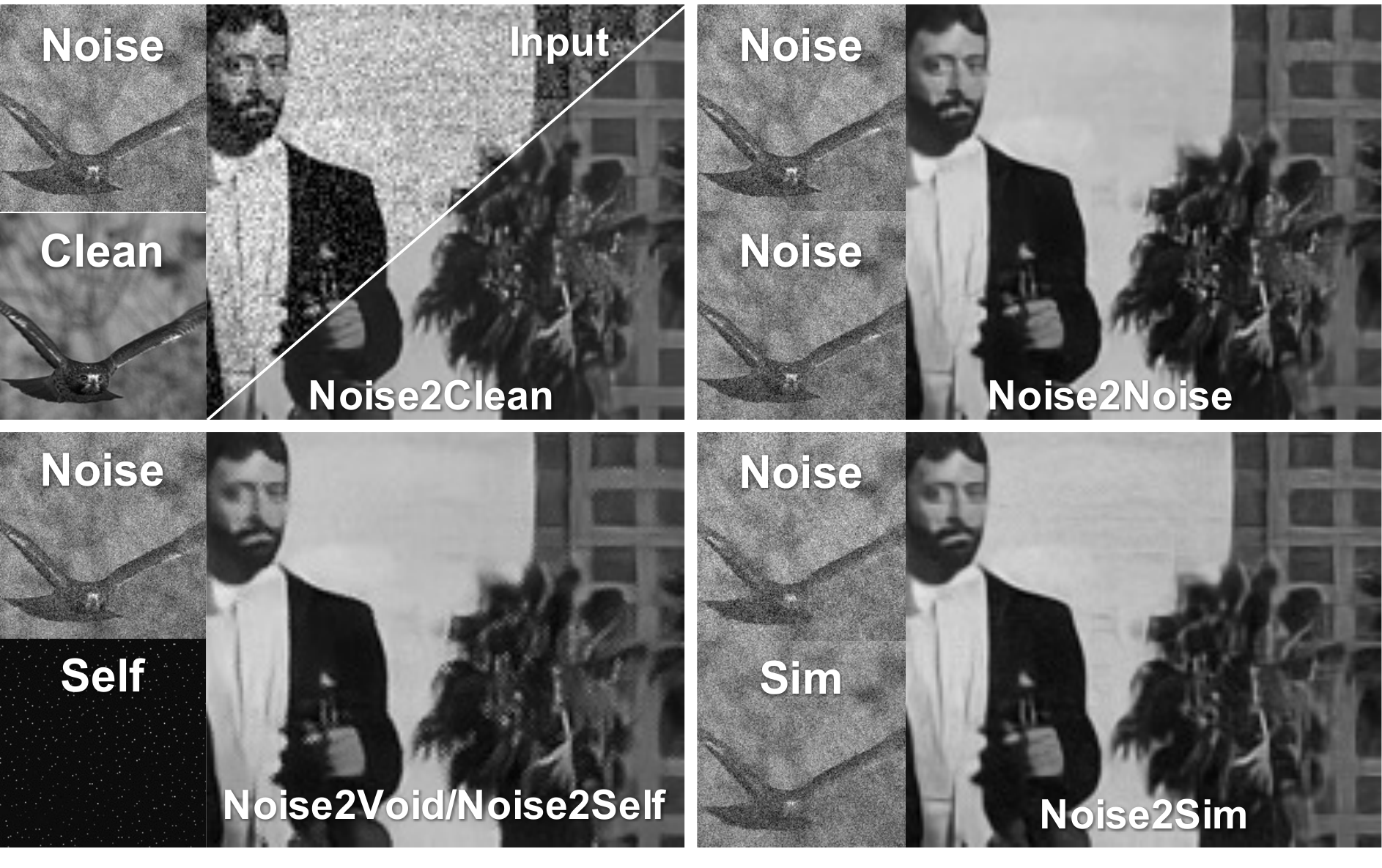}
    \caption{Comparison of different deep denoising methods. \emph{Noise2Clean} requires paired noise-clean samples. \emph{Noise2Noise} takes paired noise-noise samples, which are easier to collect than Noise2Clean counterparts but still impractical in many scenarios. \emph{Noise2Void} and \emph{Noise2Self} predict a small portion of excluded pixels based on their neighbors in single noisy images during training. Different from these methods, \emph{Noise2Sim} leverages self-similarity in a single noisy image and constructs training images using the noisy image itself in the same spirit of non-local means.}
    \label{fig:method_compare}
\end{figure}

To better understand the advantages of Noise2Sim, here we further compare Noise2Sim with Noise2Noise and Noise2void methods in terms of processing independent noise. Fig.~\ref{fig:method_compare} shows the training samples used in different deep denoising methods in the case of 2D images with independent noises.

For Noise2Noise, paired noisy images, i.e., $\bm{x}_i = \bm{s}_i + \bm{n}_i$ and $\bm{x'}_i = \bm{s}_i + \bm{n'}_i$, are required to train the network, where $\bm{n}_i$ and $\bm{n'}_i$ are two independent and zero-mean noise realizations.
The parameters $\bm{\theta_n}$ of the denoising network are optimized as

\begin{equation}
    \label{eq_loss_n}
    \bm{\theta}_n = \arg\min_{\bm{\theta}} \frac{1}{N_n} \sum_{i=1}^{N_n} || f(\s_i + \n_i; \bm{\theta}) - (\s_i + \bm{n'}_i)||_2^2,
\end{equation}
where $N_n$ is the number of paired images for Noise2Noise. Theoretically, if the conditional expectation $\mathbb{E}[\bm{n'}_i|\bm{s}_i + \bm{n}_i] = 0$ is satisfied, the learned parameters $\bm{\theta_n}$ with paired noise-noise images will be equal to $\bm{\theta_c}$ optimized under paired noise-clean images as $N_n\rightarrow \infty$~\cite{n2n_medical}. As $\hat{\bm{n}}_i$ is independent to $\bm{n}_i$, $\mathbb{E}[\bm{n'}_i|\bm{s}_i + \bm{n}_i] = \mathbb{E}[\bm{n'}_i|\bm{s}_i] = 0$ is true under the zero-mean noise assumption.
However, Noise2Sim cannot be applied when only single noisy images are available, while Noise2Sim still works well as demonstrated by our Noise2Sim theorem and experimental results.

For Noise2Void, only noisy images are required for network training. It splits the noisy image $\bm{x}_i$ into two parts, i.e., $\bm{x}_i = \bm{x}_i^c \cup \bm{x}_i^r$. As shown in Fig. \ref{fig:method_compare}, $\bm{x}_i^c$ contains a small set of pixels from $\bm{x}_i$ and used as the targets of the rest pixels $\bm{x}_i^r$. Thus, the parameters $\bm{\theta_v}$ of the denoising network are optimized as follows:
\begin{equation}
    \label{eq_loss_n}
    \bm{\theta_v} = \arg \min_{\bm{\theta}} \sum_{i=1}^{N_v} ||f(\bm{x}_i^r; \bm{\theta}) - \bm{x}_i^c||_2^2, 
\end{equation}
Three assumptions are required to make the Noise2Void network work~\cite{n2v}. First, the signal value is predictable from its local surrounding noisy pixels in the receptive field of the neural network. Second, the noise components of neighbor pixels are independent of each other. Third, the expectation of the noise component is zero, ensuring that the averaging operation gives the clean signal.

In contrast, the Noise2Sim theorem reveals that the similarity-based self-learning method is functionally equivalent to the paired learning methods Noise2Noise and Nosie2Clean.
Compared with Noise2Void and its variants, Noise2Sim has at least three advantages.
First, in terms of the training samples, we have $\bm{x}_i^r \subset \bm{x}_i$ and $\bm{x}_i^c \subset \bm{\hat{x}}_i$, so that Noise2Sim can be regarded as a non-local version of Noise2Void, leveraging both local and global information. Therefore, Noise2Sim enjoys principled superiority over Noise2Void.
Second, since Noise2Void assumes that the signal is predictable from its neighbors, it cannot preserve grainy structures or isolated pixels that are irrelevant to their neighbors and thus are not predicable.
Without relying on this signal predictability assumption, Noise2Sim does not suffer from this limitation and has the ability to preserve detailed structures as demonstrated in our experiments.
Third, current unsupervised denoising methods cannot process correlated noises due to that they aim to map between neighbor noisy pixels during training while the correlated noises are also predictable.
In contrast, Noise2Sim is a general approach that learns to map between paired similar training data, which can be constructed by adjacent or non-local pixels/patches/images/volumes. Particularly, two similar patches or sub-volumes can be far from each other within an image, the correlation between their structured noises can be ignored. Thus, Noise2Sim can effectively reduce correlated noise in various real-world applications.

\section{Estimation of Conditional Discrepancy}
\label{sec_estimate}
Here we aim to empirically estimate the values of the conditional expectation $\mathbb{E}[\bm{\delta_i} | \bm{s_i} + \bm{n}_i]$ on the BSD400 dataset, where the standard deviation of Gaussian noise was 25. Specifically, it was estimated with the following equation
\begin{equation}
\label{eq_estimate}
\mathbb{\hat{E}}[\bm{\delta_i} | \bm{s_i} + \bm{n}_i] = \frac{1}{M} \sum_{i=1}^{M} (\bm{\hat{x}'}_i - \bm{\hat{x}}''_i),
\end{equation}
where $\bm{\hat{x}'}_i$ and $\bm{\hat{x}}''_i$ are similar image pairs randomly constructed given one of BSD400 images, and $M$ was set to $5000\times 400 = 2\times 10^6$ (it means each image was repeatedly used 5,000 times), being consistent with the training process. In our experiments, we found that the values of this estimation are closed to zeros, i.e., $\mathbb{\hat{E}}[\bm{\delta_i} | \bm{s_i} + \bm{n}_i] \in [-0.0014, 0.0013]$, where $s_i$ was normalized to $[0, 1]$. Thus, the zero-mean conditional similarity condition for the Noise2Sim theorem can be well satisfied.

\section{Dataset Details}
\subsection{Natural Images}
In the section S1, we used four publicly available datasets including BSD400 and BSD68 \cite{7839189} containing grayscale images, BSD500 \cite{amfm_pami2011} and Kodak (http://r0k.us/graphics/kodak/) containing color images. Specifically, for grayscale images, we used BSD400 that consists of 400 $180 \times 180$ images for training and BSD68 that consists of 68 different sizes of images for testing. For color images, we used BSD500 that consists of 500 different image sizes for training and Kodak that consists of 24 either $768\times 512$ or $512\times 768$ images for testing.
All the above images share the pixel value range of $[0, 255]$ and were corrupted by two common noise distributions, i.e., Gaussian and Poisson.
For Gaussian noise, the noisy images were simply obtained by adding a zero-mean Gaussian noise images with any standard deviations into a clean image.
For Poisson noise, the noisy images $\bm{x}_i$ was computed as $\bm{x}_i = Poisson(\bm{s}_i/\lambda)/\lambda$, which is the same as in \cite{NIPS2019_8920}.
Note that Poisson noise is conditioned on the clean image, and the mean of noisy images is the corresponding clean image, which means a zero-mean noise.

\subsection{Low-dose CT Images}
In the main text, we used the publicly available Mayo dataset for the Low Dose CT (LDCT) Grand Challenge (https://www.aapm.org/grandchallenge/lowdosect/).
We used eight patients data for training and the other two for testing. Specifically, the training dataset consists of 4800 $512 \times 512$ slices, and the testing dataset consists of 1136 $512 \times 512$ slices.
To test the generalizability of different methods, we used two real scans of an anthropomorphic phantom from FDA (https://wiki.cancerimagingarchive.net/display/Public/Phantom+FDA).
Specifically, we used five volumetric data including 1) a low-dose (25mA) volume with \emph{b40f} kernel and a normal-dose (200mA) volume with \emph{b40f} kernel, 2) a low-dose (25mA) volume with \emph{b60f} kernel and two normal-dose (200mA) volumes with \emph{b40f} kernel, where these two normal-dose volumes contain the same contents but different noise realizations. Each of these volumes contain 408 slices with the image size of $512 \times 512$.

\subsection{Photon-counting Micro-CT Images}
As for the PCCT data used in the paper, we have scanned both a chicken leg phantom and a live mouse animal model for quantitative and qualitative experiments, respectively. The scans are performed on a commercial photon-counting micro-CT scanner (MARS, MARS Bioimaging Ltd., Christchurch, New Zealand) with a cone beam circular scanning geometry. The source is operated at 80kVp/50$\mu$A with 1.96mm Al filtration. There are 3 PCD chips stitching side by side each with $128\times 128$ pixels and 5 useful thresholds under the charge summing mode forming 5 effective energy bins, and the pixel size is $0.11\times 0.11 mm^2$. Two lateral translations have been performed during the scan to cover the $40mm$-diameter field of view, and 1440 views are collected for one rotation at each translation. The 5 energy bins are 7-20, 20-30, 30-47, 47-73 and $>$73 in keV. Due to the not negligible amount of bad pixels, iterative method is used for reconstruction. Specifically, the reconstructions are performed with 0.1mm isotropic voxels and with simultaneous algebraic reconstruction (SART) algorithm for 200 iterations. The resultant volume is of size $550 \times 550 \times 130$ for each energy bin, and 5 bins in total. 

For the quantitative experiment, a chicken leg with bones is placed in a plastic tube for consecutive scans of normal dose (300ms exposure for each view) and low dose (100ms exposure for each view). For the qualitative experiment, a live mouse bearing a tumor model and injected with blood pool contrast agent (ExiTron nano 12000, NanoPET, Berlin, Germany) is scanned with 300ms exposure for each view under anesthesia with isoflurane. Note that iterative reconstruction technique has inherent denoising effect, and usually the effect is strong at a small iteration number and gradually diminishes as the iteration number increases. On the other hand, the resolution gets enhanced (the image becomes sharper) together with the noise along with the increase of iteration. Hence, for merely SART reconstruction 200 iterations are empirically used to obtain a good balance between noise and resolution (references in Fig.~\ref{fig:pcct-leg}). For better performance with hybrid use of SART and Noise2sim, the PCCT images are first reconstructed with 600 iterations (inputs in Figs.~\ref{fig:pcct-leg} and \ref{fig:pcct-mouse}) to best reserve the resolution and then go through the network to suppress the increased noise to obtain the final clean high-resolution images.

For more results of PCCT data in section \ref{sec_morepcct}, we additionally scanned a dead mouse on another MARS spectral CT system, which includes a micro x-ray source and a flat-panel photon-counting detector (PCD) as well.
This flat-panel PCD is larger and has 660×124 pixels. The emitting x-ray spectrum at 120kVp is divided into the five energy bins: [7.0 32.0], [32.1 43], [43.1 54], [54.1 70] and [70 120]. The distances between the source to the PCD and object are 310 mm and 210 mm, respectively. In this study, we collected 5,760 views with a translation distance 88.32mm for helical scanning. The number of views for one circle is 720. The voxel size in each energy channel was set to $0.1 \times 0.1 \times 0.1 mm^3$. All reconstructed images were performed using the filtered back-projection method. The reconstruction volume is of $667 \times 394 \times 394 \times 5$, i.e., the slice size is $394 \times 394$, 667 slices in total, and 5 energy channels in use.

\section{Network Architectures}
In all our experiments, we used a simple two-layer UNet \cite{unet} with a residual connection as the denoising network in Noise2Sim, which is the same as that used in \cite{n2v}.
In RED-CNN~\cite{redcnn} and MAP-NN~\cite{shanldct}, specified network architectures are designed for LDCT denoising.

\section{Training Details}
In fast searching for similar pixels in natural images images, we used a publicly available library (https://github.com/facebookresearch/faiss), powered by GPU.
In all the experiments, we used the same data augmentation strategy as in \cite{n2v}, including random cropping followed by random $90^{\circ}$-rotation and mirroring.
The Adam \cite{adam} algorithm was coupled with the cosine learning rate schedule \cite{Loshchilov2017SGDRSG}, with the initial learning rate 0.0005.
Our method was implemented on the PyTorch (https://pytorch.org/) deep learning platform. The codes have been made available at https://github.com/niuchuangnn/noise2sim.

\section{More Denoising Results on PCCT}
\label{sec_morepcct}

\begin{figure}[h]
    \centering
    \includegraphics[width=1\textwidth]{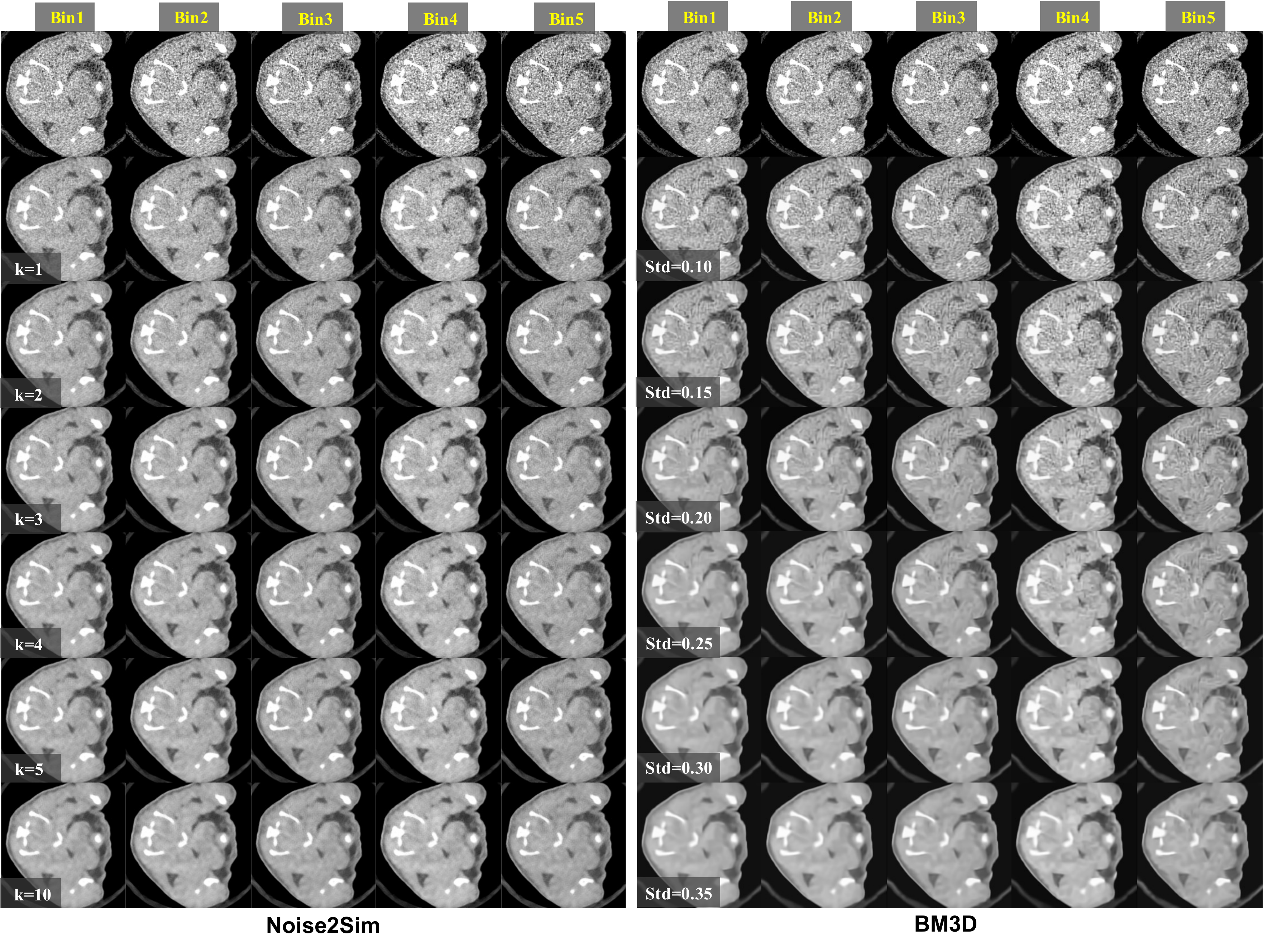}
    \caption{Noise2Sim and BM3D results on 4D spectral CT images obtained at different denoising levels.}
    \label{fig:sctall}
\end{figure}

As described in the main text, both Noise2Sim and BM3D have the ability to control the denoising level. Here we give a comprehensive visual comparison between Nosie2Sim and BM3D methods under different denoising levels, as shown in Fig.~\ref{fig:sctall}.
It can be seen that BM3D for a small denoising level generated severe structured artifacts , especially in the $5^{th}$ energy bin with strong noise. On the other hand, BM3D for a large denoising level over-smoothed image details. In contrast, Noise2Sim can remove more noise by increasing the neighborhood parameter $k$, and does not produce any artifacts and well preserves structural details.
Even for $k=10$, the subtle structures are still clear in the Noise2Sim results, without being overly smoothed as shown in the BM3D results.

\end{document}